\documentclass[11pt]{article}
\pdfoutput=1

\usepackage{lineno,lmodern} 

\usepackage[dvipsnames]{xcolor}

\def\beq{\begin{equation} }\def\eeq{\end{equation} }\def\1{\mathbf{1}}

\usepackage[framemethod=default]{mdframed}
\usepackage{caption}

\usepackage{indentfirst}
\usepackage{bm, mathrsfs, graphics,float,amssymb,amsmath,subeqnarray,setspace,graphicx,amsthm,epstopdf,subfigure, enumerate, color}
\usepackage[utf8]{inputenc}
\usepackage[colorlinks,
linkcolor=red,
anchorcolor=blue,
citecolor=blue
]{hyperref}
\usepackage{natbib}

\usepackage{fullpage}
\numberwithin{equation}{section}

\newtheorem{lemma}{Lemma}
\newtheorem{theorem}{Theorem}
\newtheorem{proposition}{Proposition}
\newtheorem{definition}{Definition}

\newtheorem{remark}{Remark}

\ifx\assumption\undefined
\newtheorem{assumption}{Assumption}
\fi

\newcommand{\cO}{\mathcal{O}}

\newcommand{\RR}{\mathbb{R}}

\newcommand{\Sb}{\mathbf{S}}
\newcommand{\Wb}{\mathbf{W}}
\newcommand{\Gb}{\mathbf{G}}

\newcommand{\bbW}{\bar{W}}

\newcommand{\bbS}{\bar{S}}

\newcommand{\bbG}{\bar{G}}
\newcommand{\bbH}{\bar{H}}

\newcommand{\TB}{\mathbb{T}}
\newcommand{\mL}{\ell}
\usepackage{bbm}

\newcommand{\tW}{\tilde{W}}
\newcommand{\tR}{\tilde{R}}

\newcommand{\orth}{\mathrm{QR}}
\newcommand{\mcs}{\mathrm{MultiConsensus}}

\usepackage{multirow}
\usepackage{tablefootnote}
\usepackage{colortbl}
\usepackage{hhline}

\usepackage{algorithm}
\usepackage{algorithmic}

\def\diag{\mathrm{diag}}
\newcommand{\norm}[1]{\left\|#1\right\|}
\newcommand{\dotprod}[1]{\left\langle #1\right\rangle}

\begin{document}
\title{
DeEPCA:
Decentralized Exact PCA with  Linear Convergence Rate 
}

\author{
Haishan Ye
\thanks{
	Shenzhen Research Institute of Big Data; The Chinese University of Hong Kong, Shenzhen;
	email: hsye\_cs@outlook.com; 
}
\and
Tong Zhang
\thanks{
	Hong Kong University of Science and Technology;
	email: tongzhang@ust.hk
}
}
\date{\today}

\maketitle

\begin{abstract}
Due to the rapid growth of smart agents such as weakly connected computational nodes and sensors, developing decentralized algorithms that can perform computations on local agents becomes a major research direction.
This paper considers the problem of decentralized Principal components analysis (PCA), which is a statistical method widely used for data analysis. We introduce a technique called subspace tracking to reduce the communication cost, and apply it  to power iterations. This leads to a decentralized PCA algorithm called \texttt{DeEPCA}, which has a convergence rate similar to that of the centralized PCA, while achieving the best communication complexity among existing decentralized PCA algorithms. \texttt{DeEPCA} is the first decentralized PCA algorithm with the number of communication rounds for each power iteration independent of target precision. Compared to existing algorithms, the proposed method is easier to tune in practice, with an improved overall communication cost. 
Our experiments validate the advantages of \texttt{DeEPCA} empirically.
\end{abstract}

\section{Introduction}

Principal Components Analysis (PCA) is a statistical data analysis method wide applications in machine learning \citep{moon2001computational,bishop2006pattern,ding2004k,dhillon2015eigenwords}, data mining \citep{cadima2004computational,lee2010super,qu2002principal}, and engineering \citep{bertrand2014distributed}.
In recent years, because of the rapid growth of data and quick advances in network technology, developing distributed algorithms has become a more and more important research topic, due to their advantages in privacy preserving, robustness, lower communication cost, etc. \citep{kairouz2019advances,lian2017can,nedic2009distributed}. 
There have been a number of previous studies of decentralized PCA algorithms \citep{scaglione2008decentralized,kempe2008decentralized,suleiman2016performance,wai2017fast}.

In a typical decentralized PAC setting, we assume that a positive semi-definite matrix $A$ is stored at different agents.
Specifically, the matrix $A$ can be decomposed as
\begin{align*}
A = \frac{1}{m}\sum_{j=1}^m A_j,
\end{align*} 
where data for $A_j$ is stored in the $j$-th agent and known only to the agent (This helps to preserve privacy). 
The agents form a connected and undirected network. Agents can communicate with their neighbors in the network to cooperatively  compute the PCA of $A$.

To obtain the top-$k$ principal components of the positive semi-definite matrix $A\in\RR^{d\times d}$, a commonly used centralized algorithm is the power method, which converges fast in practice with a linear convergence rate \citep{golub2012matrix}.
In the implementation of decentralized PCA, a natural idea is the decentralized power method (\texttt{DePM}) which mimics its centralized counterpart. 
The main procedure of \texttt{DePM} can be summarized as a local power iteration plus a multi-consensus step to synchronize the local computations \citep{kempe2008decentralized,raja2015cloud,wai2017fast,wu2018review}.
The multi-consensus step in \texttt{DePM} is used to achieve averaging.
However, decentralized PCA algorithm based on \texttt{DePM} suffers from a suboptimal communication cost, and is tricky to implement in practice. 
For each power iteration, theoretically, it requires $\cO\left(\log\frac{1}{\epsilon}\right)$ times communication, where $\epsilon$ is the target precision.
The communication cost becomes much quite significant when $\epsilon$ is small. Although seemingly only a logarithmic factor, in practice, with a data size of merely $10000$, this logarithmic factor leads to an order of magnitude more communications. 
This is clearly prohibitively large for many applications. 
Moreover, one often has to gradually increase the number of communication rounds in the multi-consensus step to deal with increased precision. However, this strategy makes the tuning of \texttt{DePM} difficult for practical applications.

In this paper, we propose a new decentralized PCA algorithm that does not suffer from the weakness of \texttt{DePM}.
We observe that the communication precision requirement in \texttt{DePM} comes from the heterogeneity of data in different agents.
Due to the heterogeneity, the local power method will converge to the top-$k$ principal components of the local matrix $A_j$ if no consensus step is conducted to perform averaging.
To conquer the weakness of \texttt{DePM} whose consensus steps in each power iteration depend on the target precision $\epsilon$, we adapted a technique called gradient tracking in the existing decentralized optimization literature, so that it can be used to track the subspace in power iterations. We call this adapted technique \emph{subspace tracking}.
Based on the subspace tracking technique and multi-consensus, we propose \underline{De}centralized \underline{E}xact \underline{PCA} (\texttt{DeEPCA}) which can achieve a linear convergence rate similar to the centralized PCA, but the consensus steps of each power iteration is independent of the target precision $\epsilon$.
We summarize our contributions as follows:
\begin{enumerate}
\item We propose a novel power-iteration based decentralized PCA called \texttt{DeEPCA}, which can achieve the \emph{best} known communication complexity, especially when the final error $\epsilon$ is small. 
Furthermore, \texttt{DeEPCA} is the first decentralized PCA algorithm whose consensus steps of each power iteration does \emph{not} depend on the target precision $\epsilon$.
\item We show that the `gradient tracking’ technique from the decentralized optimization literature can be adapted to \emph{subspace tracking} for PCA. 
The resulting \texttt{DeEPCA} algorithm can be regarded as a novel decentralized power method.
Because power method is the foundation of many matrix decomposition problems,
subspace tracking and the proof technique of \texttt{DeEPCA} can be applied to develop communication efficient decentralized algorithms for spectral analysis, and low rank matrix approximation.
\item The improvement is practically significant. Our experiments show that \texttt{DeEPCA} can achieve a linear convergence rate comparable to centralized PCA, even only a small number of consensus steps are used in each power iteration. In contrast, the conventional decentralized PCA algorithm based on \texttt{DePCA} can not converge to the principal components of $A$ when the number of consensus steps is not large.
\end{enumerate}


\section{Notation}

In this section, we introduce notations and definitions that will be used   throughout the paper.
\subsection{Notation}
Given a matrix $A=[a_{ij}] \in \RR^{n\times d}$ and a positive integer $k \le \min\{n,d\}$,  its  SVD is given as
$A=U\Sigma V^{T}=U_{k} \Sigma_{k} V_{k}^{T}+U_{\setminus k} \Sigma_{{\setminus} k} V_{{\setminus}k}^{T}$,
where $U_{k}$ and $U_{{\setminus}k}$ contain the left singular vectors of $A$,  $V_{k}$ and $V_{{\setminus}k}$ contain the right
singular vectors of $A$, and $\Sigma=\diag(\sigma_1, \ldots, \sigma_{\ell})$ with $\sigma_1\geq \sigma_2 \geq \cdots \geq \sigma_{\min\{n,d\}}\ge0$ are the nonzero singular values of $A$. 
Accordingly, we can define the Frobenius norm $\norm{A} = \sqrt{\sum_{i=1}^{\min\{n,d\}} \sigma_i^2 } = \sqrt{\sum_{i=1,j=1}^{n,d} (A(i,j))^2}$ and the spectral norm $\norm{A}_2 = \sigma_1(A)$, where $A(i,j)$ denotes the $i,j$-th entry of $A$.
We will use $\sigma_{\max}(A)$ to denote the largest singular value and $\sigma_{\min}(A)$ to denote the smallest singular value which may be zero. 
If $A$ is symmetric positive semi-definite, then it holds that $U = V$ and  $\lambda_i(A) = \sigma_i(A)$, where $\lambda_i(A)$ is the $i$-th largest eigenvalue of $A$, $\lambda_{\max}(A) = \sigma_{\max}(A)$, and $\lambda_{\min}(A) = \sigma_{\min}(A)$.

Next, we will introduce the angle between two subspaces $U\in\RR^{d\times k}$ and $X\in\RR^{d\times k}$.
\begin{definition}
Let $U \in \RR^{d \times k}$ have orthonormal columns and $X \in
\RR^{d \times k}$ have independent columns. For $V = U^\perp$, then we have 
\begin{equation}
	\label{eq:theta_def}
	\cos \theta_k(U, X)
	=  \min_{\substack{\norm{w} = 1}} \frac{\norm{U^\top X w}}{\norm{Xw}},\;
	\sin\theta_k(U,X) 
	=
	\max_{\substack{\norm{w} = 1}}  
	\frac{\norm{V^\top X w}}{\norm{Xw}}, \;\mbox{and },
	\tan \theta_k(U, X) = \max_{\norm{w} = 1
	}  \frac{\norm{V^\top X w}}{\norm{U^\top X w}}.
\end{equation}
If $X$ is orthonormal, then  it also holds that
\begin{align}
	\label{eq:theta_def_1}
	\cos \theta_k(U, X)
	=
	\sigma_{\min}(U^\top X),
	\;
	\sin\theta_k(U, X) = \norm{V^\top X}_2
	,\mbox{ and},\;
	\tan\theta_k(U,X) = \norm{V^\top X (U^\top X)^{-1}}_2,
\end{align}
where $\norm{\cdot}_2$ is the spectral norm and $\sigma_{\min}(X)$ is the smallest singular value of matrix $X$. 
\end{definition}

The above definitions can be found in the works \citep{hardt2014noisy,golub2012matrix}.

\subsection{Topology of Networks}

Let $\mathbf{L}$ be the weight matrix associated with the network, indicating
how agents are connected.
We assume that the weight matrix $\mathbf{L}$ has the following properties:
\begin{enumerate}
\item $\mathbf{L}$ is symmetric with $\mathbf{L}_{i,j} \neq 0$ if and if only agents $i$ and $j$ are connected or $i=j$.
\item ${\mathbf{0}}\preceq \mathbf{L}\preceq I$, $\mathbf{L}{\mathbf{1}}  = {\mathbf{1}}$, null($I - \mathbf{L}$) = span($\mathbf{1}$).
\end{enumerate}
We use $I$ to denote the $m\times m$ identity matrix and ${\mathbf{1}} = [1,\dots,1]^\top\in\RR^m$ denotes the vector with all ones.

The weight matrix has an important property that $\mathbf{L}^\infty = \frac{1}{m}\mathbf{1}\mathbf{1}^\top$ \citep{xiao2004fast}.
Thus, one can achieve the effect of averaging local variables on different
agents by multiple steps of local communications.
Recently, \cite{liu2011accelerated} proposed a more efficient way to achieve
averaging described in Algorithm~\ref{alg:mix} than the one in \citep{xiao2004fast}.

\begin{proposition}
\label{lem:mix_eq}
Let $\Wb^K\in\RR^{d\times d\times m}$ be the output of Algorithm~\ref{alg:mix} and $\bbW = \frac{1}{m}\Wb^0\mathbf{1} \in\RR^{d\times d}$.
Then it holds that 
\[
\bbW = \frac{1}{m}\Wb^K \mathbf{1}, \quad\mbox{and}\quad \norm{\Wb^K - \bbW\otimes\mathbf{1}}^2 \le \left(1 - \sqrt{1 - \lambda_2(\mathbf{L})}\right)^{2K} \norm{\Wb^0- \bbW\otimes\mathbf{1}}^2,
\]
where $\lambda_2(\mathbf{L})$ is the second largest eigenvalue of $\mathbf{L}$, and $\otimes$ denotes the tensor outer product. 
\end{proposition}

\section{Decentralized Exact PCA}

In this section, we propose a novel decentralized exact PCA algorithm with a linear convergence rate.
First, we provide the main idea behind our algorithm. 

\begin{algorithm}[tb]
\caption{Decentralized Exact PCA  (\texttt{DeEPCA})}
\label{alg:deepca}
\begin{small}
	\begin{algorithmic}[1]
		\STATE {\bf Input:}
		Proper initial point $W^0$, FastMix parameter $K$.
		\STATE Initialize $S_j^0 = W^0$, $W_j^0 = W^0$ and $A_j W_j^{(-1)} = W^0$.
		\FOR {$t=0,\dots, T$ }
		\STATE
		For each agent $j$, update
		\begin{align}
			\label{eq:ss}
			S_j^{t+1} = S_j^t + A_j W_j^t - A_j W_j^{t-1} 
		\end{align}
		\STATE
		Communicate $S_j^{t+1}$ with its neighbors several times to achieve averaging, that is
		\begin{align}
			\Sb^{t+1} = \mathrm{FastMix}(\Sb^{t+1}, K), \mbox{ with } \Sb^{t+1}(:,:,j) = S_j^{t+1}.
		\end{align}
		\STATE
		For each agent $j$, compute the orthonormal basis of $S_j^{t+1}$ by QR decomposition, that is
		\begin{equation}
			\label{eq:sa}
			W_j^{t+1} = \orth(S_j^{t+1}), \quad\mbox{and}\quad W_j^{t+1} = \mbox{SignAdjust}(W^{t+1}, W^0).
		\end{equation}
		\ENDFOR
		\STATE {\bf Output:} $W_j^{T+1}$
	\end{algorithmic}
\end{small}
\end{algorithm}

\begin{algorithm}[tb]
\caption{SignAdjust}
\label{alg:SA}
\begin{small}
	\begin{algorithmic}[1]
		\STATE {\bf Input:}
		Matrices $W^t$ and $W^0$ and column number $k$.
		\FOR {$i=1,\dots, k$ }
		\IF{$\dotprod{W^t(:,i), W^0(:,i)}<0$}
		\STATE Flip the sign, that is, $W^t(:,i) = -W^t(:,i)$
		\ENDIF
		\ENDFOR
		\STATE {\bf Output:} $W^t$
	\end{algorithmic}
\end{small}
\end{algorithm}

\subsection{Main Idea}

In previous works, the common algorithmic frame is to conduct a multi-consensus step to achieve averaging for each local power method \citep{raja2015cloud,wai2017fast,kempe2008decentralized}, that is,
\begin{equation}
\label{eq:dpca}
\begin{aligned}
	W_j^{t+1} =& A_j W_j^t, 
	\\
	\Wb^{t+1} =& \mcs(\Wb^{t+1}),
	\\ 
	W_j^{t+1} =&\orth(W_j^{t+1})
\end{aligned}	
\end{equation}
where  $\orth(W_j)$ computes the orthonormal basis of $W_j$ by QR decomposition and $\Wb^t\in\RR^{d\times k\times m}$ has its $j$-th slice $\Wb^t(:,:,j) = W_j^t$.
However, algorithms in this framework will take increasing consensus steps to achieve high precision principal components and the consensus steps of each power iteration depend on the target precision $\epsilon$.
This framework is similar to the well-known \texttt{DGD} algorithm in decentralized optimization which can not converges to the optima without increasing the number of communications in each multi-consensus step \citep{yuan2016convergence,nedic2009distributed}.

In decentralized optimization, to overcome the weakness of \texttt{DGD}, a novel technique called `gradient-tracking' was introduced recently \citep{qu2017harnessing,shi2015extra}.
By the advantages of the gradient-tracking, several algorithms have achieved the linear convergence rate without increasing the number of multi-consensus iterations per step.
Especially, a recent work \texttt{Mudag} showed that gradient tracking can be used to achieve a near optimal communication complexity up to a $\log$ factor \citep{ye2020multi}.

To obtain a decentralized exact PCA  algorithm with a linear convergence rate without increasing the number of communications per consensus step, we track the subspace in the proposed PCA algorithm by adapting the gradient tracking method to `subspace tracking'.
Compared with previous decentralized PCA (Eqn.~\eqref{eq:dpca}) methods, we introduce an extra term $S_j$ to track the space of power iterations. 
Combining $S_j$ with multi-consensus, we can track the subspace in the power method exactly. We can then obtain the exact principal component $W_j$ after several power iterations.
The detailed description of the resulting algorithm \texttt{DeEPCA} is in Algorithm~\ref{alg:deepca}.

Please note that, Algorithm~\ref{alg:deepca} conducts a sign adjustment in Eqn.~\eqref{eq:sa} which is necessary to make \texttt{DeEPCA} converge stably.
This is because the signs of some columns of $W_j^t$ maybe flip during the local power iterations and the sign flipping does not change the column space of the matrix. 
However, if some signs are flipped, then the outcome of the aggregation $\bbW^t = \frac{1}{m}\sum W_j^t$ will be affected.

The subspace tracking  technique in our algorithm is the key to achieving the advantages of \texttt{DeEPCA}.
The intuition behind the subspace tracking comes from the observation that when $W_j^t$ and $W_j^{t-1}$ are close to the optimal subspace $U$ (where $U$ is the top-$k$ principal components of $A$), then $A_j W_j^t - A_j W_j^{t-1}$ is close to zero. 
This implies that different local subspaces $S_j^{t+1}$ in different agents only vary by small perturbations.
Thus, we only need a small number of consensus steps to make $S_j^{t+1}$ consistent with each other.
In fact, the idea behind subspace tracking has also been used in variance reduction methods for finite sum stochastic optimization algorithms \citep{johnson2013accelerating,defazio2014saga}.

Using subspace tracking, we can maintain highly consistent subspaces $S_j^{t+1}$ in the power iteration computation $A_j W_j^t$ without increasing the number of communication rounds per consensus step.
We can show that the approximation error reduce according to $\cO(\epsilon)$ where $\epsilon$ is the error precision for the power method.

\subsection{Main Result}

The following lemma shows how the mean variable $\bbS^t = \frac{1}{m}\sum_{j=1}^m S_j^t$ converges to the top-$k$ principal components of $A$ and local variable $S_j^t$ converges to its mean counterpart $\bbS^t$.

\begin{lemma}
\label{lem:main}
Matrix $A\in\RR^{d\times d} = \frac{1}{m}\sum_{j=1}^m A_j$ is positive semi-definite with $A_j$ being stored in $j$-th agent and $\norm{A_j}_2 \le L$. 
The agents form a undirected connected graph with weighted matrix $\mathbf{L}\in\RR^{m\times m}$. Given parameter $k\ge 1$, orthonormal matrix $U\in\RR^{d\times k}$ is the top-$k$ principal components of $A$.   $\lambda_k$ and $\lambda_{k+1}$ are $k$-th and $k+1$-th largest eigenvalue of $A$, respectively.
Suppose $\mL(\bbS) \triangleq \tan\theta_k(U,\bbS)$, $\gamma = 1 - \frac{\lambda_k - \lambda_{k+1}}{2\lambda_k}$ and $\mL(\bbS^0) < \infty$. If $\rho = \left(1 - \sqrt{1- \lambda_2(\mathbf{L})}\right)^K$ satisfies
\begin{equation}
	\small
	\label{eq:rho_cond}
	\begin{aligned}
		\rho \le& \min\left\{
		\frac{\gamma}{2},
		\frac{
			(\lambda_k - \lambda_{k+1})(\lambda_k\lambda_{k+1} + 2L\lambda_{k+1})\cdot\gamma^2
		}{
			96kL(\sqrt{k}+1) \left(1+\gamma^{2t}\cdot\mL^2(\bbS^0)\right)\left(\lambda_{k+1} + 2L + (\lambda_k+2L) \gamma^{t+1}\cdot\mL(\bbS^0)\right)^2
		},
		\right.
		\\&
		\left.
		\frac{\lambda_k\lambda_{k+1} + 2L\lambda_k}
		{16Lk(\sqrt{k}+1)\sqrt{m}\gamma^{t-1}\cdot\mL(\bbS^0)
			\cdot
			\sqrt{1+\gamma^{2t}\cdot \mL^2(\bbS^0)}\left(\lambda_{k+1} + 2L + (\lambda_k+2L) \gamma^{t+1}\cdot\mL(\bbS^0)\right)
		}
		\right\},
	\end{aligned}
\end{equation}
for $t = 1,\dots, T+1$. 
Letting $\bbS^t = \frac{1}{m}\sum_{j=1}^m S_j^t $, then sequence $\{\bbS^t\}_{t=0}^{T+1}$ and $\{\Sb^t\}_{t=0}^{T+1}$ generated by Algorithm~\ref{alg:deepca}  satisfy that
\begin{align}
	\label{eq:main_dec}
	\mL(\bbS^t) \le \gamma^t \cdot \mL(\bbS^0) \quad
	\mbox{ and } \quad
	\frac{1}{\sqrt{m}} \norm{\Sb^t - \bbS^t\otimes \mathbf{1}} \le 4\rho L(\sqrt{k +1}) \gamma^{t-2}\cdot\mL(\bbS^0),
\end{align}
and
\begin{align}
	\frac{1}{\sqrt{m}} \cdot \norm{\left[\bbS^t\right]^\dagger}\norm{\Sb^t -  \bbS^t\otimes \mathbf{1}}
	\le
	\frac{
		(\lambda_k - \lambda_{k+1}) 
	}{
		24(\lambda_{k+1} + 2L)
	}
	\cdot\gamma^t\cdot\mL(\bbS^0). \label{eq:sss_ass}
\end{align}
\end{lemma}

\begin{remark}
Lemma~\ref{lem:main} shows that our \texttt{DeEPCA} can achieve a linear convergence rate almost the same to power method. 
Furthermore, the difference between local variable $S_j^t$ and its mean variable $\bbS^t$ will also converge to zero as iteration goes.
This implies that $S_j$'s in different agents will converge to the same subspace.
Thus, we can obtain that $W_j^t = \orth(S_j^t)$ will converge to the top-$k$ principal components of $A$.
Furthermore, we can observe that the right hand of Eqn.~\eqref{eq:rho_cond} decreases as $t$ increasing and is independent of $\epsilon$.
Hence, \texttt{DeEPCA} does \emph{not} require to increase the consensus steps to achieve a high precision solution nor setting consensus steps for each power iteration according to $\epsilon$ which is required in previous work \citep{wai2017fast,kempe2008decentralized}.
Lemma~\ref{lem:main} also reveals an interesting property of \texttt{DeEPCA}.
To obtain the top-$k$ principal components of a positive semi-definite matrix $A =\frac{1}{m}\sum_{j=1}^m A_j$, \texttt{DeEPCA} does \emph{not} require  $A_j$ to be positive semi-definite.
Thus, our \texttt{DeEPCA} is a robust algorithm and can be applied in different settings.
\end{remark}

By Lemma~\ref{lem:main}, we can easily obtain the iteration  and communication  complexities of \texttt{DeEPCA} to achieve $\tan\theta_k(U, W_j) \le \epsilon$ for each agent-$j$. 
The communication complexity depends on the times of local communication which is presented as the product of $\Wb$ and $\mathbf{L}$ in Algorithm~\ref{alg:mix}. 
Now we give the detailed iteration complexity and communication complexity of our algorithm in the following theorem.
\begin{theorem}
\label{thm:main}
Let $A$, $U$, and graph weight matrix $\mathbf{L}$ satisfy the properties in Lemma~\ref{lem:main}. 
The initial orthonormal matrix $W^0$ satisfies that $\tan\theta_k(U, W^0) < \infty$. 
Let parameter $K$ satisfy
\begin{align*}
	K \le 
	\frac{1}{\sqrt{1 - \lambda_2(\mathbf{L})}} 
	\cdot \log \frac{96kL(\sqrt{k}+1) (\lambda_k+2L)\left(1+\tan\theta_k(U,W^0)\right)^4 }{\lambda_{k+1}(\lambda_k - \lambda_{k+1})\cdot\left(1 - \frac{\lambda_k - \lambda_{k+1}}{2\lambda_k}\right)^2}.
\end{align*}
Given $\epsilon<1$, to achieve $\tan\theta_k(U, W_j^T) \le \epsilon$ for $j = 1,\dots, m$,
the iteration complexity $T$ is at most
\begin{align}
	T = \frac{2\lambda_k}{\lambda_k - \lambda_{k+1}} \cdot \max\left\{ \log\frac{4\tan\theta_k(U,W^0)}{\epsilon}, \log \frac{4(\lambda_k+2L)\tan\theta_k(U, W^0)}{\sqrt{m} (\lambda_k - \lambda_{k+1})\epsilon}\right\}.
\end{align}
The communication complexity is at most
\begin{equation}
	\begin{aligned}
		C =& \frac{2\lambda_k}{(\lambda_k - \lambda_{k+1})\sqrt{1 - \lambda_2(\mathbf{L})}} \cdot \max\left\{ \log\frac{4\tan\theta_k(U,W^0)}{\epsilon}, \log \frac{4(\lambda_k+2L)\tan\theta_k(U, W^0)}{\sqrt{m} (\lambda_k - \lambda_{k+1})\epsilon}\right\}
		\\
		&\cdot\log \frac{96kL(\sqrt{k}+1) (\lambda_k+2L)\left(1+\tan\theta_k(U,W^0)\right)^4}{\lambda_{k+1}(\lambda_k - \lambda_{k+1})\cdot\left(1 - \frac{\lambda_k - \lambda_{k+1}}{2\lambda_k}\right)^2}.
	\end{aligned}
\end{equation}
Furthermore, it also holds that
\begin{equation}
	\label{eq:ww}
	\norm{W_j^{T+1} - \frac{1}{m}\sum_{j=1}^m W_j^{T+1}} \le \frac{\epsilon}{2}, \mbox{ and } \tan\theta_k(U,\bbS^T) \le \frac{\epsilon}{4}.
\end{equation}
\end{theorem}

\begin{remark}
\label{rmk:main}
Theorem~\ref{thm:main} shows that for any agent $j$, $W_j$ takes $T = \cO\left(\frac{\lambda_k - \lambda_{k+1}}{\lambda_k}\log\frac{1}{\epsilon}\right)$ iterations to converge to the top-$k$ principal  components of $A$ with an $\epsilon$-suboptimality. 
This iteration complexity is the same to the centralized PCA based on power method \citep{golub2012matrix}.
Furthermore, each power iteration of \texttt{DeEPCA} requires 
\begin{equation}
	\label{eq:K}
	K = \cO\left(\frac{1}{\sqrt{1 - \lambda_2(\mathbf{L})}}\cdot \log \left(\frac{L^2}{\lambda_k\lambda_{k+1}} \cdot \frac{\lambda_k - \lambda_{k+1}}{\lambda_k}\right)\right)
\end{equation} 
consensus steps.
Note that $K$ is independent of the precision parameter $\epsilon$ which shows that \texttt{DeEPCA} does \emph{not} need to tune its consensus parameter $K$ according to $\epsilon$.
This also implies that \texttt{DeEPCA} does \emph{not} increase its consensus steps gradually to achieve a high precision principal components. 
In contrast, the best known consensus steps for each power iteration of previous decentralized algorithms are required to be \citep{wai2017fast}
\begin{equation}
	\label{eq:K1}
	K = \cO\left(\frac{1}{\sqrt{1 - \lambda_2(\mathbf{L})}} \log\left(\frac{\lambda_k - \lambda_{k+1}}{\lambda_k}\cdot \frac{1}{\epsilon}\right)\right).
\end{equation}
Thus, \texttt{DeEPCA} achieves the best communication complexity of decentralized PCA algorithms.
Comparing Eqn.~\eqref{eq:K} and~\eqref{eq:K1}, our result is better than the one of \citep{wai2017fast} up to $\log\frac{1}{\epsilon}$ factor. 
In fact, this advantage will become large even when $\epsilon$ is moderate large which can be observed in our experiments.
Similar advantage of \texttt{EXTRA} over \texttt{DGD} in decentralized optimization makes \texttt{EXTRA} become one of most important algorithm in decentralized optimization \citep{shi2015extra}.

Furthermore, Eqn.~\eqref{eq:K} shows that the consensus steps depend on the ratio $L^2/(\lambda_k\lambda_{k+1})$. 
In fact, the value $L^2/(\lambda_k\lambda_{k+1})$ reflect the data heterogeneity which can be observed more clearly when $k = 1$.
Due the data heterogeneity, multi-consensus is necessary in \texttt{DeEPCA} which will be validated in our experiments. 
\end{remark}

\begin{remark}
Lemma~\ref{lem:main} shows that once $\rho$ satisfies Eqn.~\eqref{eq:rho_cond}, $\bbS^t$ will converge to the top-$k$ principal components of $A$ linearly.
That, any multi-consensus which can satisfy Eqn.~\eqref{eq:rho_cond}, \texttt{DeEPCA} can achieve linear convergence rate.
Thus, though our analysis is based on the undirected graph, the results of \texttt{DeEPCA} can be easily extended to directed graph, gossip models, etc.
\end{remark}

\begin{remark}
\texttt{DeEPCA} is a novel decentralized exact power method.
Because the  power method is the key tool in eigenvector computation and low rank approximation (SVD decomposition) \citep{golub2012matrix},
\texttt{DeEPCA} provides a solid foundation for developing decentralized eigenvalue decomposition, decentralized SVD,   decentralized spectral analysis, etc.
\end{remark}
\section{Convergence Analysis}

In this section, we will give the detailed convergence analysis of \texttt{DeEPCA}.
For notation convenience, we first introduce local and aggregate variables.

\subsection{Local and Aggregate Variables}
\label{subsec:lav}

Matrix $W^t_j\in\RR^{d\times k}$ is the local copy of the variable of $W$ for agent $j$ at $t$-th power iteration and we introduce its aggregate variable $\Wb^t\in\RR^{d\times k\times m}$ whose $j$-th slice $\Wb^t(:,:,j)$ is $W_j^t$, that is,
\begin{align*}
\Wb^t(:,:,j) = W_j^t.
\end{align*}
Furthermore, we introduce $G_j^t = A_j W_j^{t-1} \in\RR^{d\times k}$ and tracking variable $S_j\in\RR^{d\times k}$.  
We also introduce the aggregate variables $\Gb^t\in\RR^{d\times k\times m}$ and $\Sb^t \in\RR^{d\times k\times m}$ of $G_i^t$ and $S_i^t$, respectively which satisfy
\begin{align}
\label{eq:GS}
\Gb^t(:,:,j) = G_j^t\;\mbox{ and }\; \Sb^t(:,:,j) = S_j^t.
\end{align}
Using the local and aggregate variables, we can represent Algorithm~\ref{alg:deepca} as
\begin{align}
\Sb^{t+1} =& \mathrm{FastMix} \left(\Sb^t + \Gb^{t+1} - \Gb^t, K\right) \label{eq:SS}
\\
W_j^{t+1} =& \orth(S_j^{t+1}). \label{eq:orth}
\end{align}

For the convergence analysis, we further introduce the mean values 
\begin{align}
\label{eq:def_bb}
\bbW^t = \frac{1}{m}\sum_{j=1}^m W_j^t, \quad 
\bbG^t = \frac{1}{m}\sum_{j=1}^m G_j^t, \quad
\bbS^t = \frac{1}{m}\sum_{j=1}^m S_j^t, \quad
\bbH^t = \frac{1}{m}\sum_{j=1}^m A_i\bbW^{t-1},\quad
\tW^t = \orth(\bbS^t).
\end{align}

\subsection{Sketch of Proof}

First, we give the relationship between $\bbS^t$, $\bbG^t$, and $\bbH^t$ in Lemma~\ref{lem:s_g} and Lemma~\ref{lem:g_h}. 
These two lemmas show that $\bbS^t$ and $\bbH^t$ are close to each other but perturbed by $\frac{L}{\sqrt{m}}\norm{\Wb^{t-1} -  \bbW^{t-1}\otimes \mathbf{1}}$.
Furthermore, by the definition of $\bbH^t$, we can obtain that
\begin{align}
\label{eq:bbs}
\bbS^{t+1} \approx \bbH^{t+1} = A\bbW^t.
\end{align}
If $\bbS^t$ is also close to $S_j^t$, then we can obtain that 
\begin{align}
\label{eq:bbw}
\bbW^{t+1} \approx \orth(\bbS^{t+1}).
\end{align}
We can observe that Eqn.~\eqref{eq:bbs} and \eqref{eq:bbw} are the two steps of a power iteration but with some perturbation.
Based on $\bbS^t$, $\bbG^t$, and $\bbH^t$, we can observe that \texttt{DeEPCA} can fit into the framework of power method but with some perturbation.
This is the reason why \texttt{DeEPCA} will converge to the top-$k$ principal components of $A$.

Next, we will bound the error between local and mean variables (Defined in Section~\ref{subsec:lav}) such as $\norm{\Sb^{t+1} - \bbS^{t+1}\otimes \mathbf{1}}$ (in Lemma~\ref{lem:S_s}) and $\norm{\Wb^{t+1} - \bbW^{t+1}\otimes\mathbf{1}}$ (in Lemma~\ref{lem:W_w}).
Lemma~\ref{lem:S_s} shows that $\norm{\Sb^{t+1} -  \bbS^{t+1}\otimes \mathbf{1}}$ will decay with a rate $\rho<1$ for each iteration but adding an extra error term $L\rho \norm{\Wb^t - \Wb^{t-1}}$.
When \texttt{DeEPCA} converges, then $\Wb^t$ and $\Wb^{t-1}$ will both converge to the top-$k$ principal components, that is, $ \norm{\Wb^t - \Wb^{t-1}}$ will converge to zero (in Lemma~\ref{lem:W-W}).
Thus, $\norm{\Sb^{t+1} -  \bbS^{t+1}\otimes \mathbf{1}}$ will also converge to zero.
This implies that $\norm{\Wb^t - \bbW^t\otimes \mathbf{1}}$ goes to zero as $t$ increases by Lemma~\ref{lem:W_w}.
Hence, the noisy power method described in Eqn.~\eqref{eq:bbs} and~\eqref{eq:bbw} becomes exact power method gradually.

Finally, Lemma~\ref{lem:mL_dec} shows that $\tan\theta_k(U,\bbS^t)$ converges with rate $\gamma = 1 - \frac{\lambda_k-\lambda_{k+1}}{2\lambda_k}$ when the perturbation term $ \norm{\left[\bbS^t\right]^\dagger}\norm{\Sb^t -  \bbS^t\otimes \mathbf{1}}$ is upper bounded as Eqn.~\eqref{eq:s_ass}.
Combining Lemma~\ref{lem:S_s}, Lemma~\ref{lem:bbs_norm} and Lemma~\ref{lem:mL_dec}, we use induction in the proof of Lemma~\ref{lem:main} to show that the assumption~\eqref{eq:s_ass} and Eqn.~\eqref{eq:mL_dec} hold for $t = 1,\dots, T+1$ when $\rho$ is properly chosen. 
This leads to the results of Lemma~\ref{lem:main}.
\subsection{Main Lemmas}
\label{subsec:main_lem}

In our analysis, we aim to show 
$
\tan\theta_k(U, \bbS^{T+1})
$
and 
$\norm{\Sb^{T+1} - \bbS^T\otimes \mathbf{1}} $ will converge to $\epsilon$.
First, we give the relationship between $\bbS^t$, $\bbG^t$, and $\bbH^t$. 
Based on $\bbS^t$, $\bbG^t$, and $\bbH^t$, we can observe that \texttt{DeEPCA} can fit into the framework of power method but with some perturbation.
\begin{lemma}
\label{lem:s_g}
Let $\bbW^0$, $\bbG^0$, and $\bbS^0$ be initialized as $W^0$. 
Supposing $\bbG^t$, and $\bbS^t$ be defined in Eqn.~\eqref{eq:def_bb} and $\Sb^{t}$ update as Eqn.~\eqref{eq:SS},
it holds that 
$$\bbS^{t+1} = \bbS^t + \bbG^{t+1} - \bbG^t =  \bbG^{t+1}.$$
\end{lemma}

\begin{lemma}
\label{lem:g_h}
Letting  $\bbG^t$ and $\bbH^t$ be defined in Eqn.~\eqref{eq:def_bb} and $\norm{A_j}_2 \le L$ for $j = 1,\dots, m$, they have the following properties
\begin{equation}
	\label{eq:g_h}
	\norm{\bbG^t - \bbH^t} 
	\le
	\frac{L}{\sqrt{m}}\norm{\Wb^{t-1} -  \bbW^{t-1}\otimes \mathbf{1}}.
\end{equation}
\end{lemma}

In the next lemmas, we will bound the error between local and mean variables (Defined in Section~\ref{subsec:lav}).
First, we upper bound the error $\norm{\Sb^{t+1} -  \bbS^{t+1}\otimes \mathbf{1}}$ recursively.

\begin{algorithm}[tb]
\caption{FastMix}
\label{alg:mix}
\begin{small}
	\begin{algorithmic}[1]
		\STATE {\bf Input:} $\Wb^{0} = \Wb^{-1}$, $K$, $\mathbf{L}$, step size $\eta_w = \frac{1 - \sqrt{1-\lambda_2^2(W)}}{1+\sqrt{1-\lambda_2^2(W)}}$.
		\FOR {$k=0,\dots, K$ }
		\STATE $\Wb^{k+1} = (1+\eta_w)\Wb^k\mathbf{L} - \eta_w \Wb^{k-1}$;  
		\ENDFOR
		\STATE {\bf Output:} $\Wb^K$.
	\end{algorithmic}
\end{small}
\end{algorithm}

\begin{lemma}
\label{lem:S_s}
Letting $\Sb^{t}$ be updated as Eqn.~\eqref{eq:SS} and $\norm{A_j} \le L$, then $\Sb^{t+1}$ and $\bbS^{t+1}$ have the following properties
\begin{align}
	\label{eq:S_s}
	\norm{\Sb^{t+1} -  \bbS^{t+1}\otimes \mathbf{1}}
	\le&
	\rho \norm{\Sb^t -  \bbS^t\otimes \mathbf{1}} + L\rho\norm{\Wb^t - \Wb^{t-1}}, \mbox{ with } 
	\rho \triangleq \left(1 - \sqrt{1 - \lambda_2(\mathbf{L})}\right)^K.
\end{align}
\end{lemma}

\begin{lemma}
\label{lem:bbs_norm}
If for $t = 0,1,\dots,t$, it holds that
$
\sigma_{\min}(U^\top \tW^t) > 0 
$
with $\tW$ defined in Eqn.~\eqref{eq:def_bb} and $U$ being top principal components of $A$,
then we can obtain that
\begin{equation}
	\label{eq:bbs_norm}
	\sigma_{\min}(\bbS^{t+1}) 
	\ge
	\lambda_k\cdot \frac{1}{\sqrt{1+ \mL^2(\bbS^t)}}
	-
	\frac{24L}{\sqrt{m}} \norm{\left[\bbS^t\right]^\dagger}\norm{\Sb^t -  \bbS^t\otimes \mathbf{1}}.
\end{equation}
\end{lemma}

Now, we will bound the error $\norm{\Wb^t - \bbW^t\otimes\mathbf{1}} $.
\begin{lemma}
\label{lem:W_w}
Assuming that $\norm{\left[\bbS^t\right]^\dagger} \norm{\bbS^t - S_j^t} \le \frac{1}{4}$ for $j = 1,\dots, m$, where $\left[\bbS^t\right]^\dagger$ is the pseudo inverse of $\bbS^t$, then it holds that
\begin{align}
	\norm{\Wb^t - \bbW^t\otimes\mathbf{1}} 
	\le 
	12\norm{\left[\bbS^t\right]^\dagger}\norm{\Sb^t -  \bbS^t\otimes \mathbf{1}}. \label{eq:W_w}
\end{align}
Letting $\bbS^t = \tW^t \tR^t$ be the QR decomposition of $\bbS^t$, then it holds that
\begin{align}
	\norm{\tW^t - \bbW^t}
	\le 
	\frac{12}{\sqrt{m}}\norm{\left[\bbS^t\right]^\dagger}\norm{\Sb^t -  \bbS^t\otimes \mathbf{1}}. \label{eq:w_s}
\end{align}
\end{lemma}

Next, we will give the convergence rate of $\bbS^t$ under the assumption that the error between local variable $S_j^t$ and its mean counterpart $\bbS^t$ is upper bounded.
\begin{lemma}
\label{lem:mL_dec}
Letting $\mL(\bbS) \triangleq \tan\theta_k(U,\bbS)$, $\gamma \triangleq 1 - \frac{\lambda_k - \lambda_{k+1}}{2\lambda_k}$ and  $\frac{1}{\sqrt{m}}\cdot \norm{\left[\bbS^t\right]^\dagger}\norm{\Sb^t -  \bbS^t\otimes \mathbf{1}}$ satisfy
\begin{align}
	\label{eq:s_ass}
	\frac{1}{\sqrt{m}} \cdot \norm{\left[\bbS^t\right]^\dagger}\norm{\Sb^t -  \bbS^t\otimes \mathbf{1}}
	\le
	\frac{
		(\lambda_k - \lambda_{k+1}) \cdot\gamma^t\cdot\mL(\bbS^0)
	}{
		24\sqrt{1+\gamma^{2t}\cdot\mL^2(\bbS^0)}\left(\lambda_{k+1}+2L + (\lambda_k+2L) \gamma^{t+1}\mL(\bbS^0)\right)
	},
\end{align}
for $t = 0,1,\dots,T$, sequence $\{\bbS^t\}$ generated by Algorithm~\ref{alg:deepca} satisfies 
\begin{equation}
	\label{eq:mL_dec}
	\mL(\bbS^t) \le \gamma^{t+1}\cdot \mL(\bbS^0).
\end{equation}
\end{lemma}

Finally, we will bound the difference between $\Wb^t$ and $\Wb^{t-1}$.

\begin{lemma}
\label{lem:W-W}
Letting $\Wb$ be defined in Eqn.~\eqref{eq:def_bb} and $\mL(\bbS) \triangleq \tan\theta_k(U,\bbS)$, then it holds that
\begin{equation}
	\label{eq:W-W}
	\begin{aligned}
		\norm{\Wb^t - \Wb^{t-1}}
		\le&
		24\left(
		\norm{\left[\bbS^t\right]^\dagger}\norm{\Sb^t -  \bbS^t\otimes \mathbf{1}}
		+
		\norm{\left[\bbS^{t-1}\right]^\dagger}\norm{\Sb^{t-1} -  \bbS^{t-1}\otimes \mathbf{1}}
		\right)
		\\&
		+
		\sqrt{mk}\cdot \left(
		\mL(\bbS^t )
		+ 
		\mL(\bbS^{t-1})
		\right).
	\end{aligned}	
\end{equation}
\end{lemma}
\subsection{Proof of Main Results}

Using lemmas in previous subsection, we can prove Lemma~\ref{lem:main} and Theorem~\ref{thm:main} as follows.

\begin{proof}[Proof of Lemma~\ref{lem:main}]
We prove the result by induction.
When $t=0$, Eqn.~\eqref{eq:s_ass} holds since each agent shares the same initialization. 
This implies that $\mL(\bbS^1) \le \gamma \cdot \mL(\bbS^0)$.

Now, we assume that Eqn.~\eqref{eq:s_ass} and~\eqref{eq:mL_dec} hold for $t = 0,\dots,T$.
In this case, for $t = 1,\dots, T$, it holds that
\begin{align*}
	\mL(\bbS^t) \le \gamma^t \cdot \mL(\bbS^0),
\end{align*}
and
\begin{align*}
	\frac{1}{\sqrt{m}} \cdot \norm{\left[\bbS^t\right]^\dagger}\norm{\Sb^t -  \bbS^t\otimes \mathbf{1}}
	\le&
	\frac{
		(\lambda_k - \lambda_{k+1}) \cdot\gamma^t\cdot\mL(\bbS^0)
	}{
		24\sqrt{1+\gamma^{2t}\cdot\mL^2(\bbS^0)}\left(\lambda_{k+1}+2L + (\lambda_k+2L) \gamma^{t+1}\mL(\bbS^0)\right)
	}
	\\
	\le&
	\frac{
		(\lambda_k - \lambda_{k+1}) 
	}{
		24(\lambda_{k+1} + 2L)
	}
	\cdot\gamma^t\cdot\mL(\bbS^0). \label{eq:ss_ass}
\end{align*}
We will show that the result holds for $t = T+1$ and we only need to prove Eqn.~\eqref{eq:s_ass} will hold for $t = T+1$.
First, by Eqn.~\eqref{eq:S_s}, we have
\begin{align*}
	&\norm{\Sb^{t+1} - \bbS^{t+1}\otimes \mathbf{1}}
	\\
	\overset{\eqref{eq:S_s}}{\le}&
	\rho 
	\norm{\Sb^t - \bbS^t\otimes \mathbf{1}}
	+
	\rho L\norm{\Wb^t - \Wb^{t-1}}
	\\
	\overset{\eqref{eq:W-W}}{\le}&
	\rho 
	\norm{\Sb^t - \bbS^t\otimes \mathbf{1}}
	+
	\rho L \sqrt{mk} \cdot \left(\mL(\bbS^t) + \mL(\bbS^{t-1})\right)
	\\&
	+ 24\rho L\left(
	\norm{\left[\bbS^t\right]^\dagger}\norm{\Sb^t -  \bbS^t\otimes \mathbf{1}}
	+
	\norm{\left[\bbS^{t-1}\right]^\dagger}\norm{\Sb^{t-1} -  \bbS^{t-1}\otimes \mathbf{1}}
	\right)
	\\
	\overset{\eqref{eq:ss_ass}}{\le}&
	\rho 
	\norm{\Sb^t - \bbS^t\otimes \mathbf{1}}
	+
	\rho L \sqrt{mk} \cdot \left(\mL(\bbS^t) + \mL(\bbS^{t-1})\right)
	+
	24\rho L \sqrt{m} \cdot \frac{\lambda_k - \lambda_{k+1}}{24(\lambda_{k+1} + 2L)} \left(\gamma^t + \gamma^{t-1}\right) \cdot  \mL(\bbS^0)
	\\
	\le&
	\rho \norm{\Sb^t - \bbS^t\otimes \mathbf{1}}
	+
	2\rho L \sqrt{m} (\sqrt{k}+1) \gamma^{t-1} \cdot \mL(\bbS^0),
\end{align*}
which implies that
\begin{align*}
	\frac{1}{\sqrt{m}} \norm{\Sb^{t+1} - \bbS^{t+1}\otimes \mathbf{1}}
	\le
	\rho \cdot \frac{1}{\sqrt{m}}\norm{\Sb^t - \bbS^t\otimes \mathbf{1}}
	+
	2\rho L(\sqrt{k}+1) \gamma^{t-1} \cdot \mL(\bbS^0).
\end{align*}
Using above equation recursively, we can obtain that
\begin{align*}
	\frac{1}{\sqrt{m}} \norm{\Sb^{T+1} - \bbS^{T+1}\otimes \mathbf{1}}
	\le& 
	\rho^{T+1} \cdot \frac{1}{\sqrt{m}}\norm{\Sb^0 - \bbS^0\otimes \mathbf{1}}
	+
	2\rho L (\sqrt{k}+1) \cdot\mL(\bbS^0) \sum_{i=1}^{T}\rho^{T-i} \gamma^i
	\\
	=&
	2\rho L (\sqrt{k}+1) \cdot\mL(\bbS^0) \cdot \frac{\gamma^T-\gamma\rho^T}{\gamma - \rho}
	\\
	\le&
	4\rho L(\sqrt{k}+1) \gamma^{T-1} \cdot \mL(\bbS^0),
\end{align*}
where the first equality is because each agent shares the same initialization and last inequality is because of the assumption that $\rho  \le \frac{\gamma}{2}$.  

Furthermore, we have
\begin{align*}
	\sigma_{\min}(\bbS^{T+1})
	\overset{\eqref{eq:bbs_norm}}{\ge}&
	\frac{\lambda_k}{\sqrt{1+\mL^2(\bbS^T)}} 
	-
	\frac{24L}{\sqrt{m}} \norm{\left[\bbS^T\right]^\dagger}\norm{\Sb^T -  \bbS^T\otimes \mathbf{1}}
	\\
	\overset{\eqref{eq:s_ass}}{\ge}&
	\frac{\lambda_k}{\sqrt{1+\gamma^{2T} \cdot \mL^2(\bbS^0)}}
	-
	\frac{L(\lambda_k - \lambda_{k+1})\gamma^T\cdot \mL(\bbS^0)}{\sqrt{1+\gamma^{2T}\cdot \mL^2(\bbS^0)}(\lambda_{k+1} + 2L + (\lambda_k+2L) \gamma^{T+1}\cdot\mL(\bbS^0))}
	\\
	=&\frac{\lambda_k\lambda_{k+1} + 2L\lambda_k + \left(\frac{\lambda_k(\lambda_k+\lambda_{k+1})}{2} + 2L\lambda_{k+1}\right) \gamma^T\cdot\mL(\bbS^0)}{\sqrt{1+\gamma^{2T}\cdot \mL^2(\bbS^0)}(\lambda_{k+1} + 2L + (\lambda_k+2L) \gamma^{T+1}\cdot\mL(\bbS^0))}.
\end{align*}
Therefore, we can obtain that
\begin{align*}
	&\frac{1}{\sqrt{m}}\norm{\left[\bbS^{T+1}\right]^\dagger}\norm{\Sb^{T+1} -  \bbS^{T+1}\otimes \mathbf{1}}
	\\
	\le&
	k \cdot \frac{
		\sqrt{1+\gamma^{2T}\cdot \mL^2(\bbS^0)}(\lambda_{k+1} + 2L + (\lambda_k+2L) \gamma^{T+1}\cdot\mL(\bbS^0))
	}{\lambda_k\lambda_{k+1} + 2L\lambda_k + \left(\frac{\lambda_k(\lambda_k+\lambda_{k+1})}{2} + 2L\lambda_{k+1}\right) \gamma^T\cdot\mL(\bbS^0)}
	\cdot
	4\rho L(\sqrt{k}+1) \gamma^{T-1} \cdot \mL(\bbS^0).
\end{align*}
First, we need to satisfy the condition in Lemma~\ref{lem:W_w}, that is,
\begin{align*}
	\norm{\left[\bbS^{T+1}\right]^\dagger}\norm{S_j^{T+1} -  \bbS^{T+1}}
	\le
	\norm{\left[\bbS^{T+1}\right]^\dagger}\norm{\Sb^{T+1} -  \bbS^{T+1}\otimes \mathbf{1}}
	\le\frac{1}{4}.
\end{align*}
Therefore, $\rho$ only needs
\begin{align*}
	\rho 
	\le
	\frac{1}{16Lk(\sqrt{k}+1)\sqrt{m}\gamma^{T-1}\cdot\mL(\bbS^0)} 
	\cdot
	\frac{
		\lambda_k\lambda_{k+1} + 2L\lambda_k + \left(\frac{\lambda_k(\lambda_k+\lambda_{k+1})}{2} + 2L\lambda_{k+1}\right) \gamma^T\cdot\mL(\bbS^0)
	}{
		\sqrt{1+\gamma^{2T}\cdot \mL^2(\bbS^0)}(\lambda_{k+1} + 2L + (\lambda_k+2L) \gamma^{T+1}\cdot\mL(\bbS^0))
	}.
\end{align*} 
To simplify above equation, we only require $\rho$ to be
\begin{align*}
	\rho 
	\le
	\frac{\lambda_k\lambda_{k+1} + 2L\lambda_k}
	{16Lk(\sqrt{k}+1)\sqrt{m}\gamma^{T-1}\cdot\mL(\bbS^0)
		\cdot
		\sqrt{1+\gamma^{2T}\cdot \mL^2(\bbS^0)}\left(\lambda_{k+1} + 2L + (\lambda_k+2L) \gamma^{T+1}\cdot\mL(\bbS^0)\right)
	} .
\end{align*}
To satisfy Eqn.~\eqref{eq:s_ass} for $t = T+1$, $\rho$ only needs to satisfy
\begin{align*}
	\rho
	\le&
	\frac{
		(\lambda_k - \lambda_{k+1}) \cdot\gamma^2
	}{
		96kL(\sqrt{k}+1)\sqrt{1+\gamma^{2(T+1)}\cdot\mL^2(\bbS^0)}\left(\lambda_{k+1}+2L + (\lambda_k+2L) \gamma^{T+2}\mL(\bbS^0)\right)
	}
	\\
	&\cdot
	\frac{
		\lambda_k\lambda_{k+1} + 2L\lambda_k + \left(\frac{\lambda_k(\lambda_k+\lambda_{k+1})}{2} + 2L\lambda_{k+1}\right) \gamma^T\cdot\mL(\bbS^0)
	}{
		\sqrt{1+\gamma^{2T}\cdot \mL^2(\bbS^0)}(\lambda_{k+1} + 2L + (\lambda_k+2L) \gamma^{T+1}\cdot\mL(\bbS^0))
	}.
\end{align*}
To simplify above equation, we only require $\rho$ to be
\begin{align*}
	\rho 
	\le
	\frac{
		(\lambda_k - \lambda_{k+1})(\lambda_k\lambda_{k+1} + 2L\lambda_{k+1})\cdot\gamma^2
	}{
		96kL(\sqrt{k}+1) \left(1+\gamma^{2T}\cdot\mL^2(\bbS^0)\right)\left(\lambda_{k+1} + 2L + (\lambda_k+2L) \gamma^{T+1}\cdot\mL(\bbS^0)\right)^2
	}.
\end{align*}
Since Eqn.~\eqref{eq:s_ass} holds for $t = T+1$ when $\rho$ satisfies the condition~\eqref{eq:rho_cond}, then Eqn.~\eqref{eq:mL_dec} also holds for $t = T+1$.
This concludes the proof.
\end{proof}

Using the results of Lemma~\ref{lem:main}, we can prove Theorem~\ref{thm:main} as follows.
\begin{proof}[Proof of Theorem~\ref{thm:main}]
First, by Eqn.~\eqref{eq:W_w}, Eqn.~\eqref{eq:sss_ass}, and the condition that $T \ge \frac{2\lambda_k}{\lambda_k - \lambda_{k+1}}\log \frac{4(\lambda_k+2L)\tan\theta_k(U, W^0)}{\sqrt{m} (\lambda_k - \lambda_{k+1})\epsilon}$, we can obtain that
\begin{align*}
	\norm{\Wb^T - \bbW^T\otimes\mathbf{1}} 
	\le
	\sqrt{m} \cdot \frac{\lambda_k - \lambda_{k+1}}{2(\lambda_{k+1} + 2L)}\cdot \gamma^T\cdot\mL(\bbS^0) \le \frac{\epsilon}{2}.
\end{align*}
Similarly, we can obtain that $\tan\theta_k(U,\bbS^T) \le \frac{\epsilon}{4}$.
Thus, we can obtain the results in Eqn.~\eqref{eq:ww}.

Furthermore, by the definition of angels between two subspaces, we have
\begin{align*}
	\tan\theta_k(U, W_j^t) 
	\overset{\eqref{eq:theta_def}}{=}&
	\max_{\norm{w} = 1}\frac{\norm{V^\top W_j^T w}}{\norm{U^\top W_j^T w}}
	\\
	\le&
	\max_{\norm{w} = 1}\frac{\norm{V^\top \bbW^T w} + \norm{W_j^T - \bbW^T}}{\norm{U^\top \bbW^T w} - \norm{W_j^T - \bbW^T}}
	\\
	\le&
	\max_{\norm{w} = 1}\frac{\norm{V^\top \tW^T w} + \norm{\tW^T - \bbW^T}+ \norm{W_j^T - \bbW^T}}{\norm{U^\top \tW^T w} - \norm{\tW^T - \bbW^T} - \norm{W_j^T - \bbW^T}}
	\\
	\overset{\eqref{eq:W_w},\eqref{eq:w_s}}{\le}&
	\max_{\norm{w} = 1}\frac{\norm{V^\top \tW^T w} + 24\norm{\left[\bbS^T\right]^\dagger}\norm{\Sb^T -  \bbS^T\otimes \mathbf{1}}}{\norm{U^\top \tW^T w} -24\norm{\left[\bbS^T\right]^\dagger}\norm{\Sb^T -  \bbS^T\otimes \mathbf{1}}}
	\\
	=&
	\frac{
		\tan\theta_k(U,\tW^T) + 24\norm{\left[\bbS^T\right]^\dagger}\norm{\Sb^T -  \bbS^T\otimes \mathbf{1}} / \cos\theta_k(U,\tW^T)}{
		1 - 24 \norm{\left[\bbS^T\right]^\dagger}\norm{\Sb^T -  \bbS^T\otimes \mathbf{1}} / \cos\theta_k(U,\tW^T)}
	\\
	\overset{\eqref{eq:main_dec},\eqref{eq:sss_ass}}{\le}&
	\frac{
		\gamma^T\cdot \mL(\bbS^0) + \sqrt{m}\cdot\frac{\lambda_k - \lambda_{k+1}}{\lambda_k + 2L} \cdot \gamma^T\cdot\mL(\bbS^0) \cdot \sqrt{1+\gamma^{2T}\cdot \mL^2(\bbS^0)}
	}{
		1 - \sqrt{m}\cdot\frac{\lambda_k - \lambda_{k+1}}{\lambda_k + 2L} \cdot \gamma^T\cdot\mL(\bbS^0) \cdot \sqrt{1+\gamma^{2T}\cdot \mL^2(\bbS^0)}
	}
	\\
	=& 
	\frac{
		\gamma^T\cdot \tan\theta_k(U,W^0) + \sqrt{m}\cdot\frac{\lambda_k - \lambda_{k+1}}{\lambda_k + 2L} \cdot \gamma^T\cdot\tan\theta_k(U,W^0) \cdot \sqrt{1+\gamma^{2T}\cdot \tan^2\theta_k(U,W^0)}
	}{
		1 - \sqrt{m}\cdot\frac{\lambda_k - \lambda_{k+1}}{\lambda_k + 2L} \cdot \gamma^T\cdot\mL(\bbS^0) \cdot \sqrt{1+\gamma^{2T}\cdot \mL^2(\bbS^0)}
	}
\end{align*}
Since $T = \frac{2\lambda_k}{\lambda_k - \lambda_{k+1}}\log \frac{4\tan\theta_k(U, W^0)}{\epsilon}$, it holds that
$
\gamma^T\cdot \tan\theta_k(U,W^0) 
\le
\frac{\epsilon}{4}.
$
Furthermore, when $T = \frac{2\lambda_k}{\lambda_k - \lambda_{k+1}}\log \frac{4(\lambda_k+2L)\tan\theta_k(U, W^0)}{\sqrt{m} (\lambda_k - \lambda_{k+1})\epsilon}$, 
it holds that $
\sqrt{m}\cdot\frac{\lambda_k - \lambda_{k+1}}{\lambda_k + 2L} \cdot \gamma^T\cdot\tan\theta_k(U,W^0)
\le
\frac{\epsilon}{4}.
$
Thus, when $\epsilon < 1$, we can obtain that
\begin{align*}
	\tan\theta_k(U, W_j^T) 
	\le
	\frac{
		\epsilon/4+\epsilon/4 \cdot \sqrt{1+1/4^2}
	}{
		1- 1/4\cdot \sqrt{1 + 1/4^2}}
	<
	\epsilon.
\end{align*}

Since the right hand of Eqn.~\ref{eq:rho_cond} is monotone deceasing as $t$ increases, $\rho$ only satisfies that
\begin{equation*}
	\begin{aligned}
		\rho \le& \min\left\{
		\frac{
			(\lambda_k - \lambda_{k+1})(\lambda_k\lambda_{k+1} + 2L\lambda_{k+1})\cdot\gamma^2
		}{
			96kL(\sqrt{k}+1) \left(1+\mL^2(\bbS^0)\right)\left(\lambda_{k+1} + 2L + (\lambda_k+2L) \cdot\mL(\bbS^0)\right)^2
		},
		\right.
		\\&
		\left.
		\frac{\lambda_k\lambda_{k+1} + 2L\lambda_k}
		{16Lk(\sqrt{k}+1)\sqrt{m}\cdot\mL(\bbS^0)
			\cdot
			\sqrt{1+ \mL^2(\bbS^0)}\left(\lambda_{k+1} + 2L + (\lambda_k+2L) \cdot\mL(\bbS^0)\right)
		}
		\right\}.
	\end{aligned}
\end{equation*}
Furthermore, $\rho$ only requires to satisfy
\begin{align}
	\label{eq:rho_co}
	\rho 
	\le& 
	\frac{
		(\lambda_k - \lambda_{k+1})(\lambda_k\lambda_{k+1} + 2L\lambda_{k+1})\cdot\gamma^2
	}{
		96kL(\sqrt{k}+1) \left(1+\mL^2(\bbS^0)\right)\left(\lambda_k + 2L + (\lambda_k+2L) \cdot\mL(\bbS^0)\right)^2
	}
\end{align}
Replacing the definition of $\mL(\bbS^0)$ and Proposition~\ref{lem:mix_eq}, we can obtain if $K$ satisfies that
\begin{align*}
	K \le 
	\frac{1}{\sqrt{1 - \lambda_2(\mathbf{L})}} 
	\cdot \log \frac{96kL(\sqrt{k}+1) (\lambda_k+2L)\left(1+\tan\theta_k(U,W^0)\right)^4 }{\lambda_{k+1}(\lambda_k - \lambda_{k+1})\cdot\left(1 - \frac{\lambda_k - \lambda_{k+1}}{2\lambda_k}\right)^2},
\end{align*}
the requirement of $\rho$ in Eqn.~\eqref{eq:rho_co} is satisfied.
Combining with iteration complexity, we can obtain the total communication complexity
\begin{align*}
	C =& T\times K = \frac{2\lambda_k}{(\lambda_k - \lambda_{k+1})\sqrt{1 - \lambda_2(\mathbf{L})}} \cdot \max\left\{ \log\frac{4\tan\theta_k(U,W^0)}{\epsilon}, \log \frac{4(\lambda_k+2L)\tan\theta_k(U, W^0)}{\sqrt{m} (\lambda_k - \lambda_{k+1})\epsilon}\right\}
	\\
	&\cdot\log \frac{96kL(\sqrt{k}+1) (\lambda_k+2L)\left(1+\tan\theta_k(U,W^0)\right)^4}{\lambda_{k+1}(\lambda_k - \lambda_{k+1})\cdot\left(1 - \frac{\lambda_k - \lambda_{k+1}}{2\lambda_k}\right)^2}.
\end{align*}
\end{proof}

\begin{figure*}[th!]
\subfigtopskip = 0pt
\begin{center}
	\centering
	\subfigure[$\norm{\Sb - \bbS\otimes \mathbf{1}}$ with $K = 3$]{\includegraphics[width=50mm]{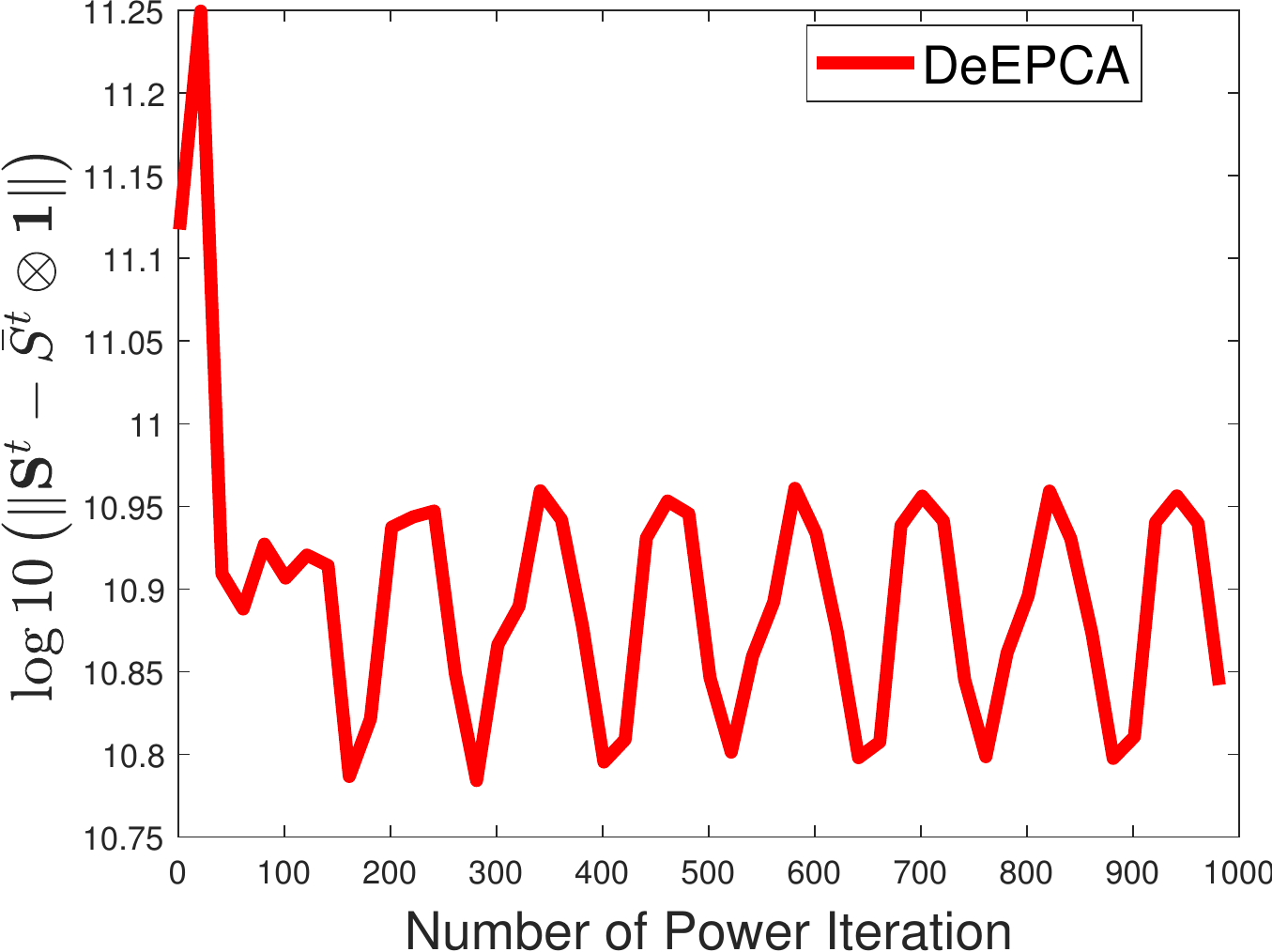}}~
	\subfigure[$\norm{\Wb - \bbW\otimes \mathbf{1}}$ with $K=3$]{\includegraphics[width=50mm]{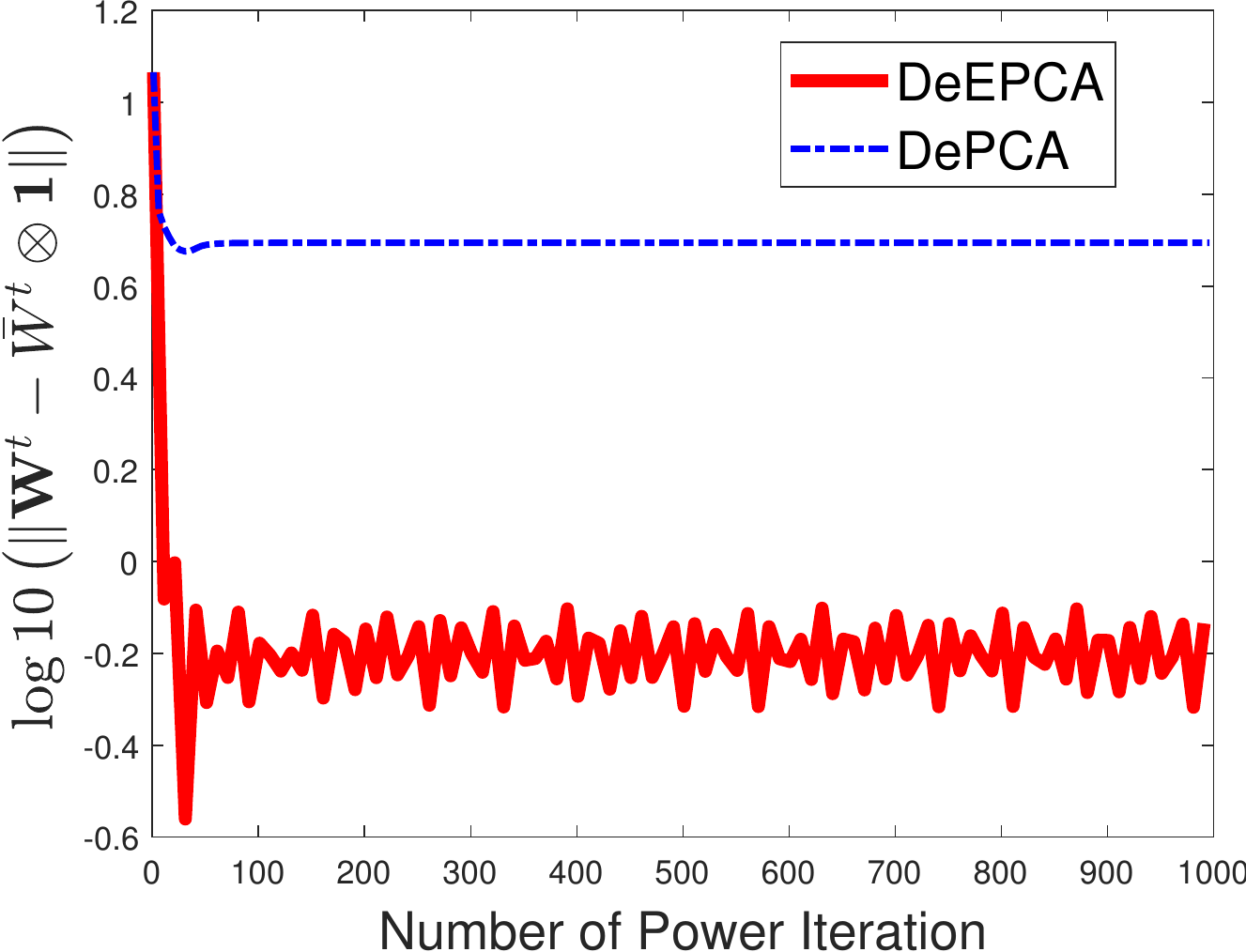}}~
	\subfigure[$\tan\theta_k(U, W)$ with $K=3$]{\includegraphics[width=50mm]{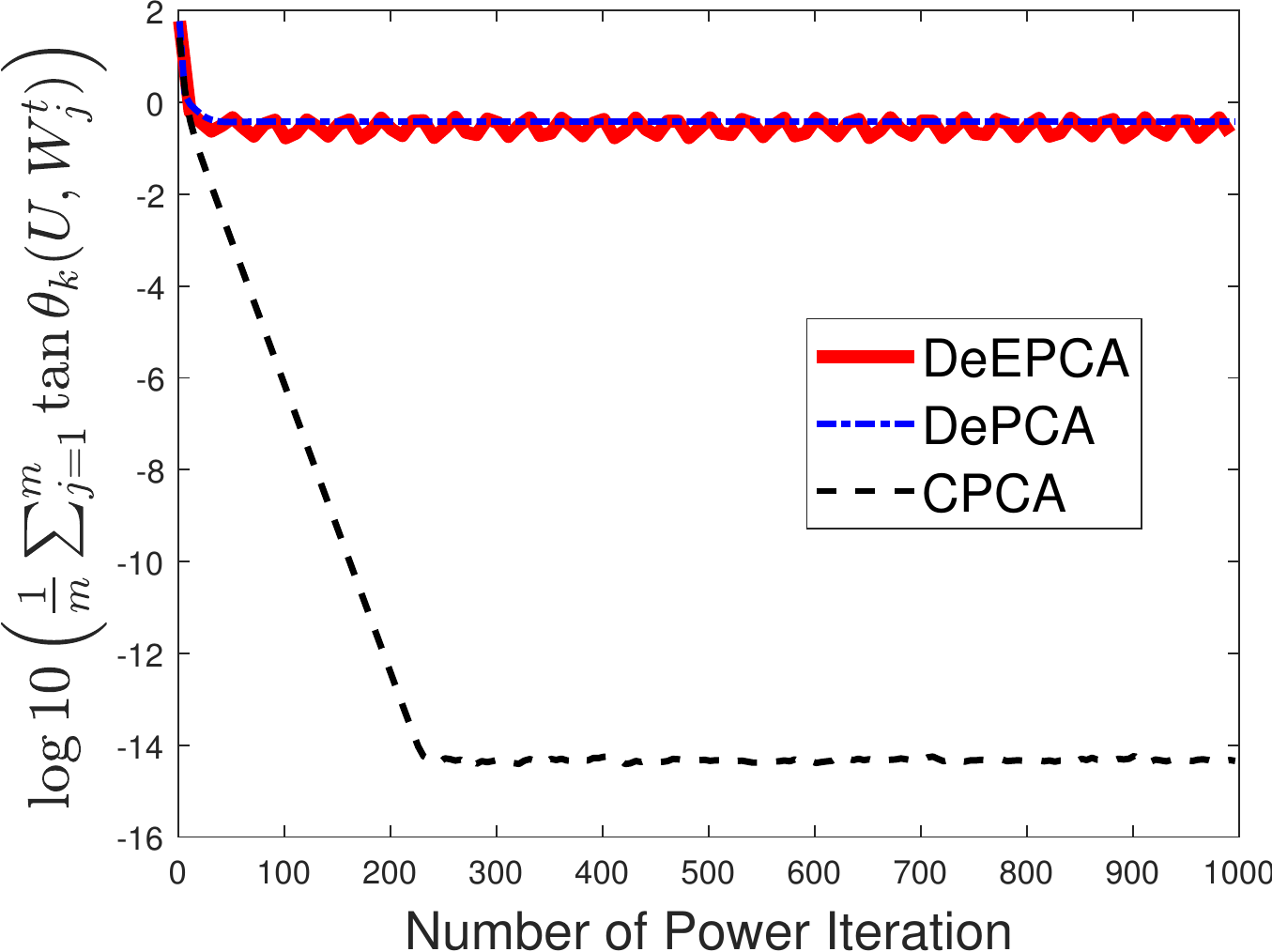}}
	\\
	\subfigure[$\norm{\Sb - \bbS\otimes \mathbf{1}}$ with $K = 5$]{\includegraphics[width=50mm]{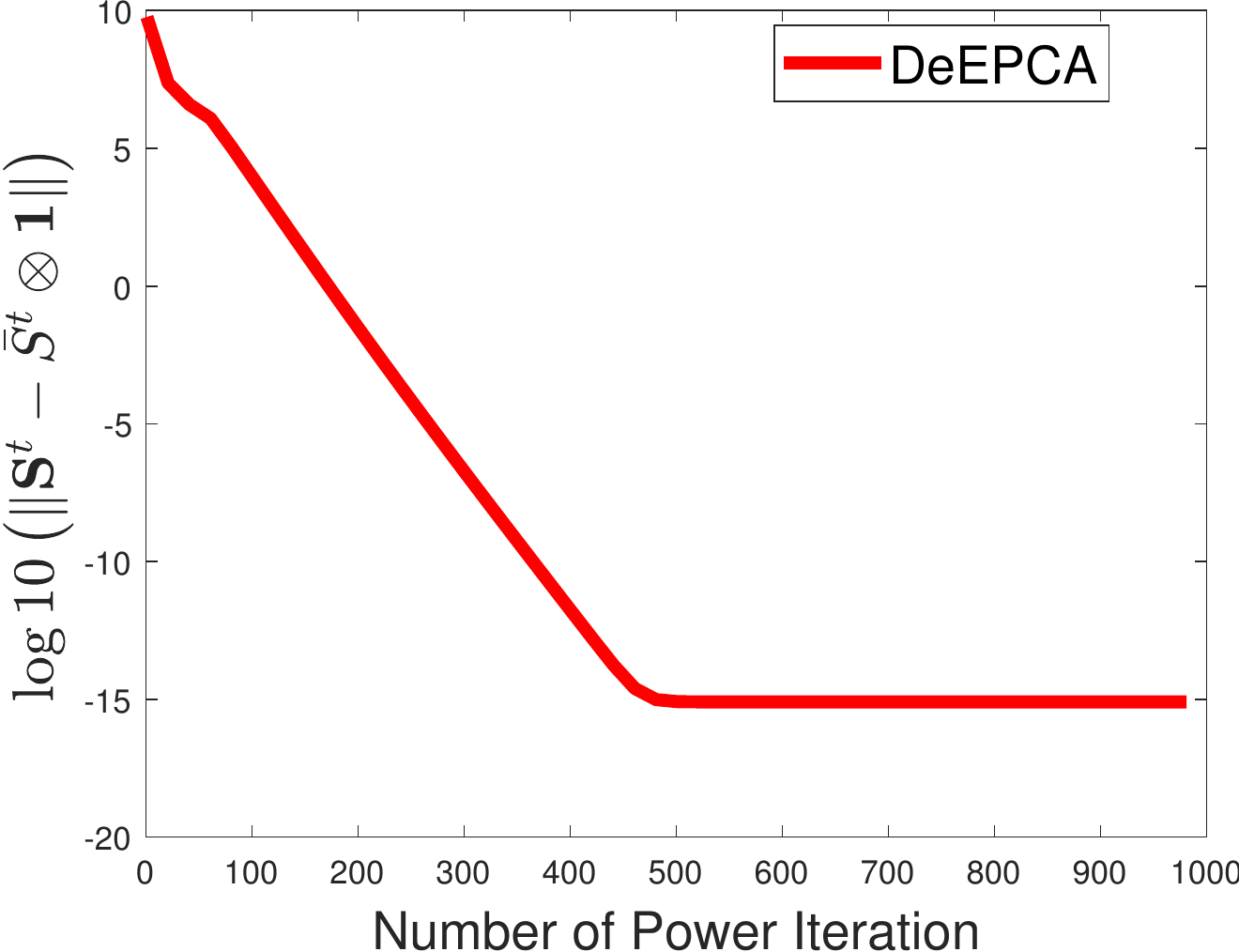}}~
	\subfigure[$\norm{\Wb - \bbW\otimes \mathbf{1}}$ with $K=5$]{\includegraphics[width=50mm]{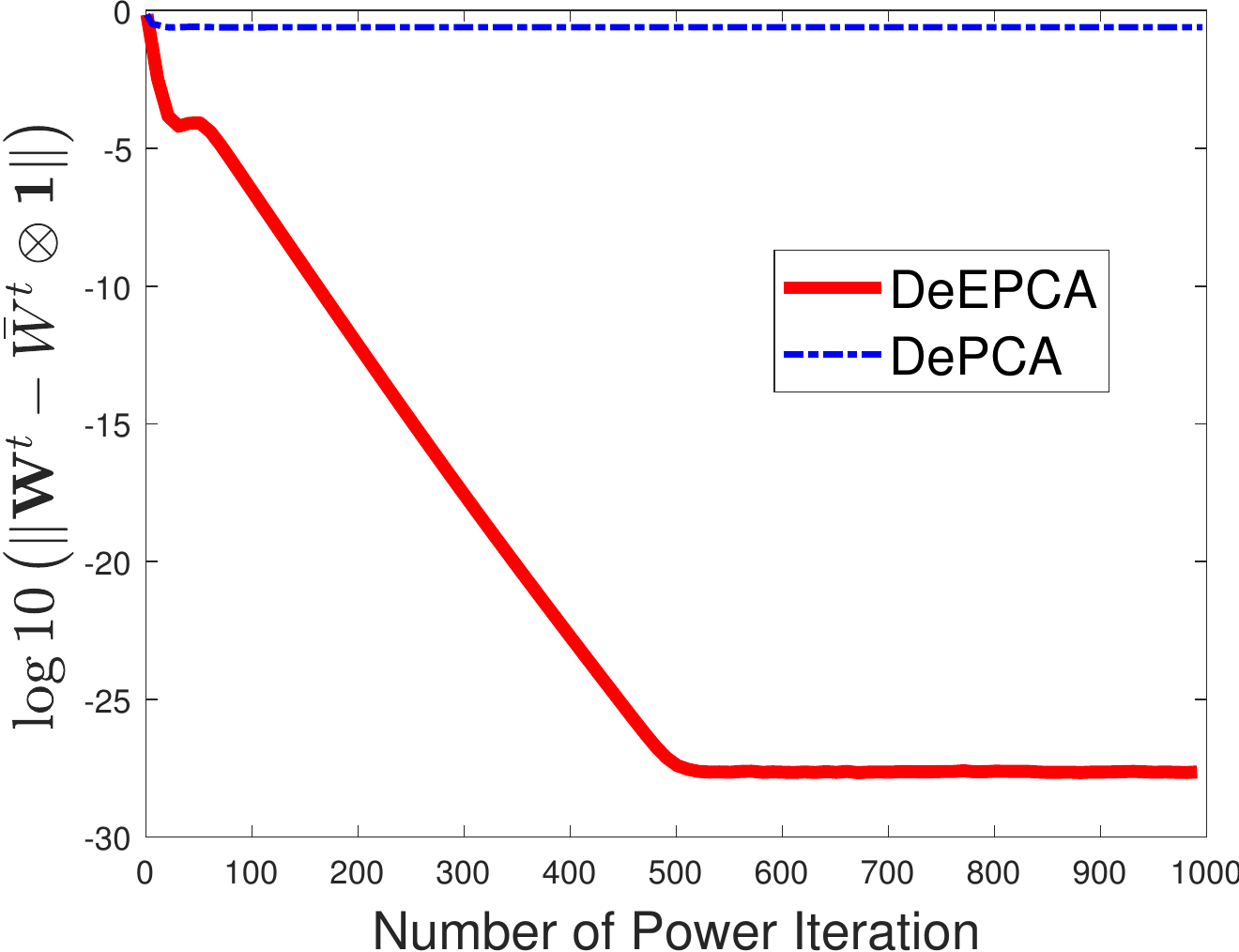}}~
	\subfigure[$\tan\theta_k(U, W)$ with $K=5$]{\includegraphics[width=50mm]{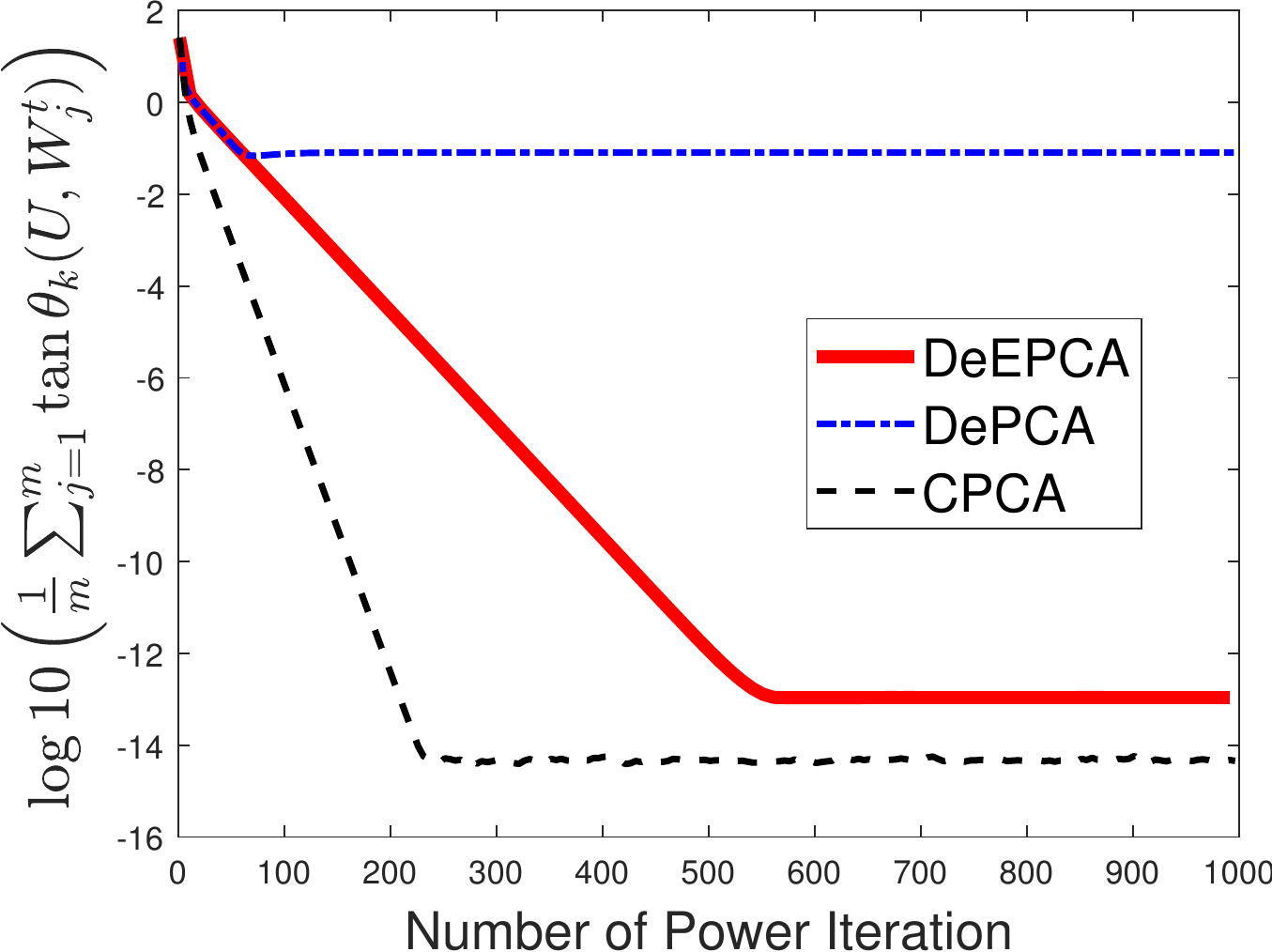}}
	\\
	\subfigure[$\norm{\Sb - \bbS\otimes \mathbf{1}}$ with $K = 10$]{\includegraphics[width=50mm]{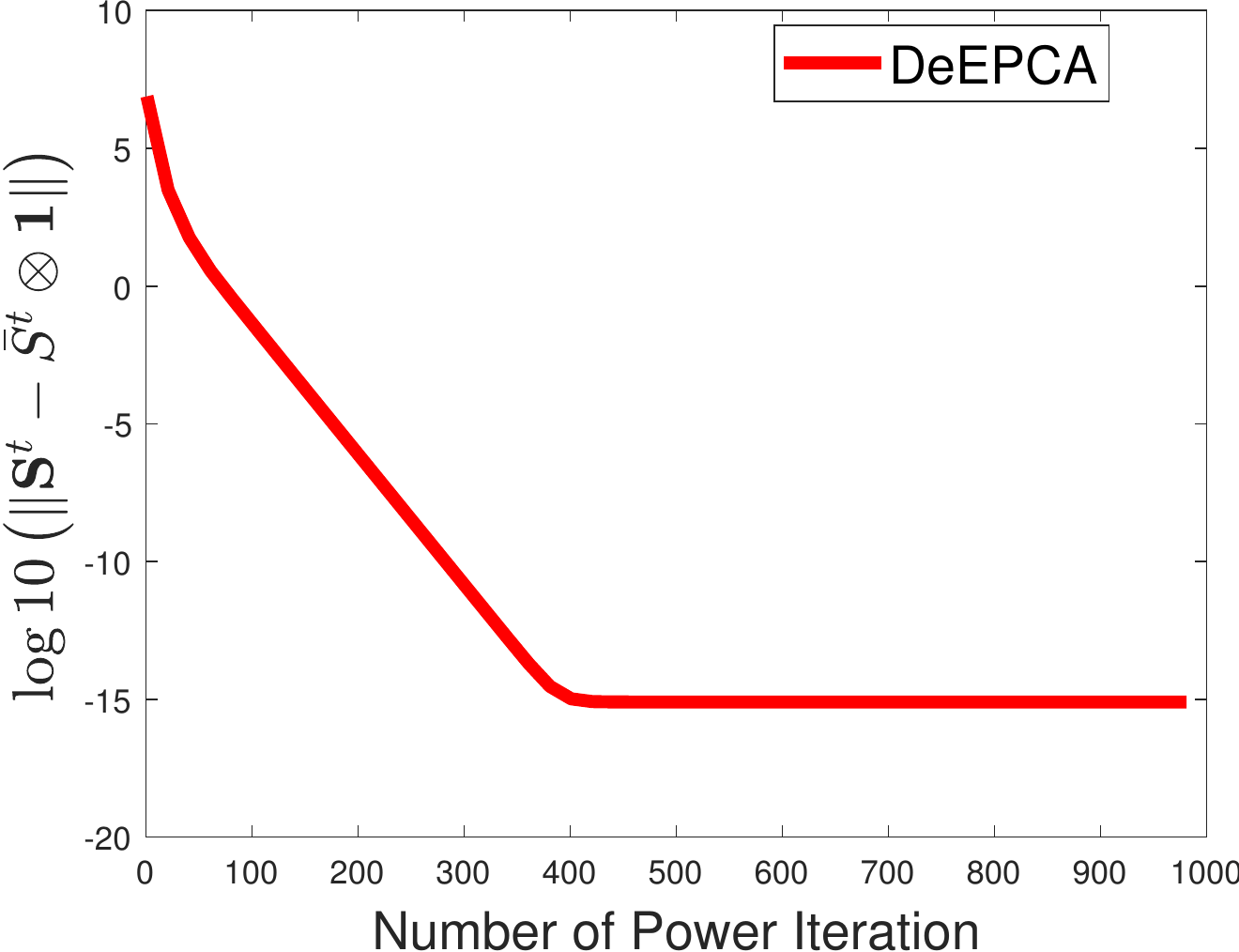}}~
	\subfigure[$\norm{\Wb - \bbW\otimes \mathbf{1}}$ with $K=10$]{\includegraphics[width=50mm]{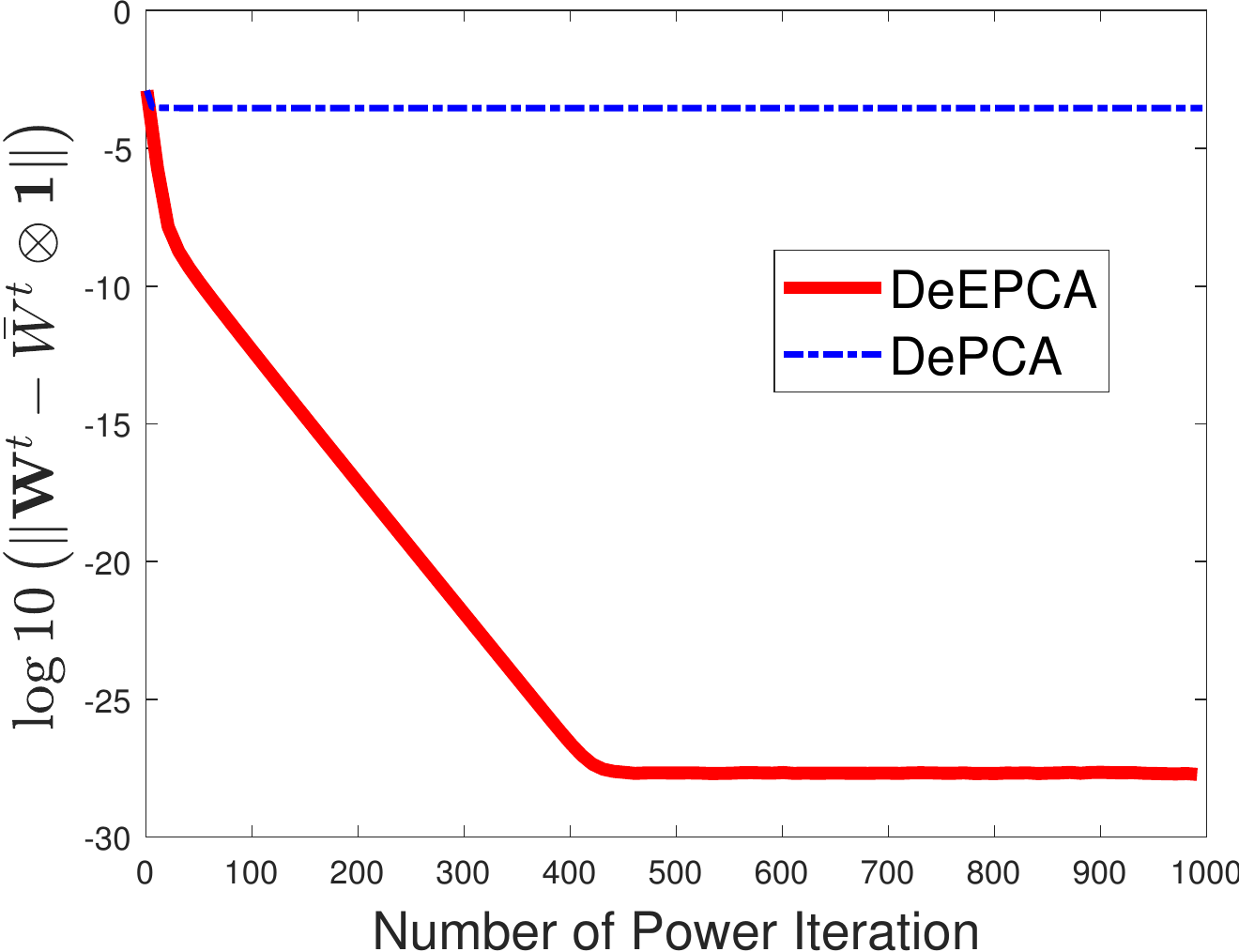}}~
	\subfigure[$\tan\theta_k(U, W)$ with $K=10$]{\includegraphics[width=50mm]{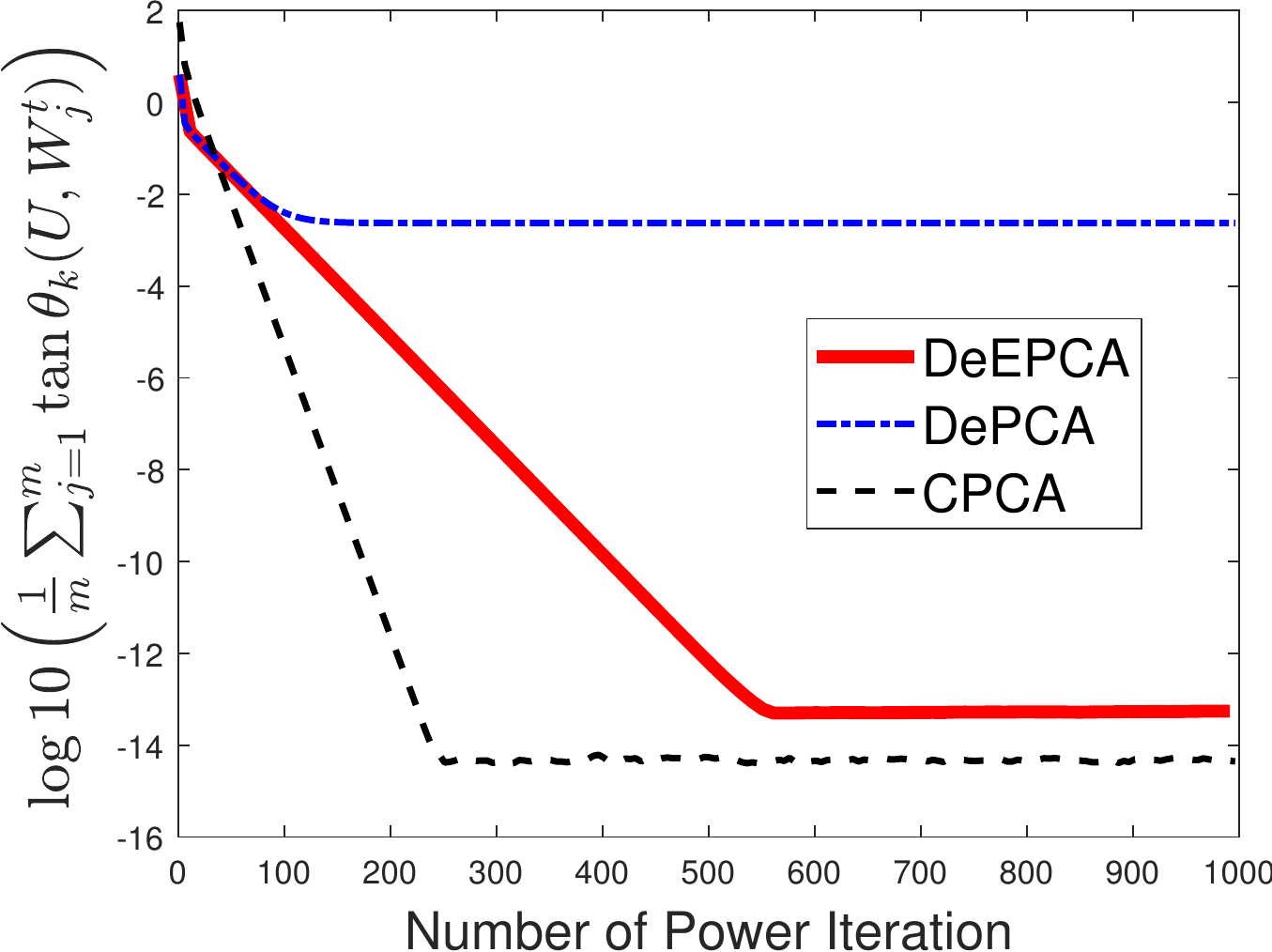}}
\end{center}
\vskip -0.1in
\caption{Experiment on `w8a'.}
\label{fig:w8a}
\end{figure*}

\section{Experiments}

In the previous sections, we presented a theoretical analysis of our algorithm. 
In this section, we will provide empirical studies.

\paragraph{Experiment Setting}
In our experiments, we consider random networks where each pair of agents
has a connection with a probability of $p = 0.5$. We set $\mathbf{L} = I - \frac{M}{\lambda_{\max}(M)}
$
where $M$ is the Laplacian matrix associated with a weighted graph. 
We set $m = 50$ ,
that is, there exists $50$ agents in this network. In our experiments, the gossip matrix $\mathbf{L}$ satisfies $1 - \lambda_2 (\mathbf{L}) = 0.4563$.

We conduct experiments on the datasets ‘w8a’ and ‘a9a’ which can be downloaded in libsvm datasets.
For ‘w8a’, we set $n = 800$ and $d = 300$. For ‘a9a’, we set $n = 600$ and $d = 123$.
For each agent, $A_j$ has the following form
\begin{align}
A = \frac{1}{m}\sum_{j=1}^m A_j, \mbox{ and }\;	A_j = \sum_{i=1}^n v_i v_i^\top, \mbox{ with }\; v_i = a_{(j-1)*n+i},
\end{align} 
where $a_{(j-1)*n+i}\in\RR^d$ is the $((j-1)*n+i)$-th input vector of the dataset.

\paragraph{Experiment Results}

\begin{figure*}[th!]
\subfigtopskip = 0pt
\begin{center}
	\centering
	\subfigure[$\norm{\Sb - \bbS\otimes \mathbf{1}}$ with $K = 1$]{\includegraphics[width=50mm]{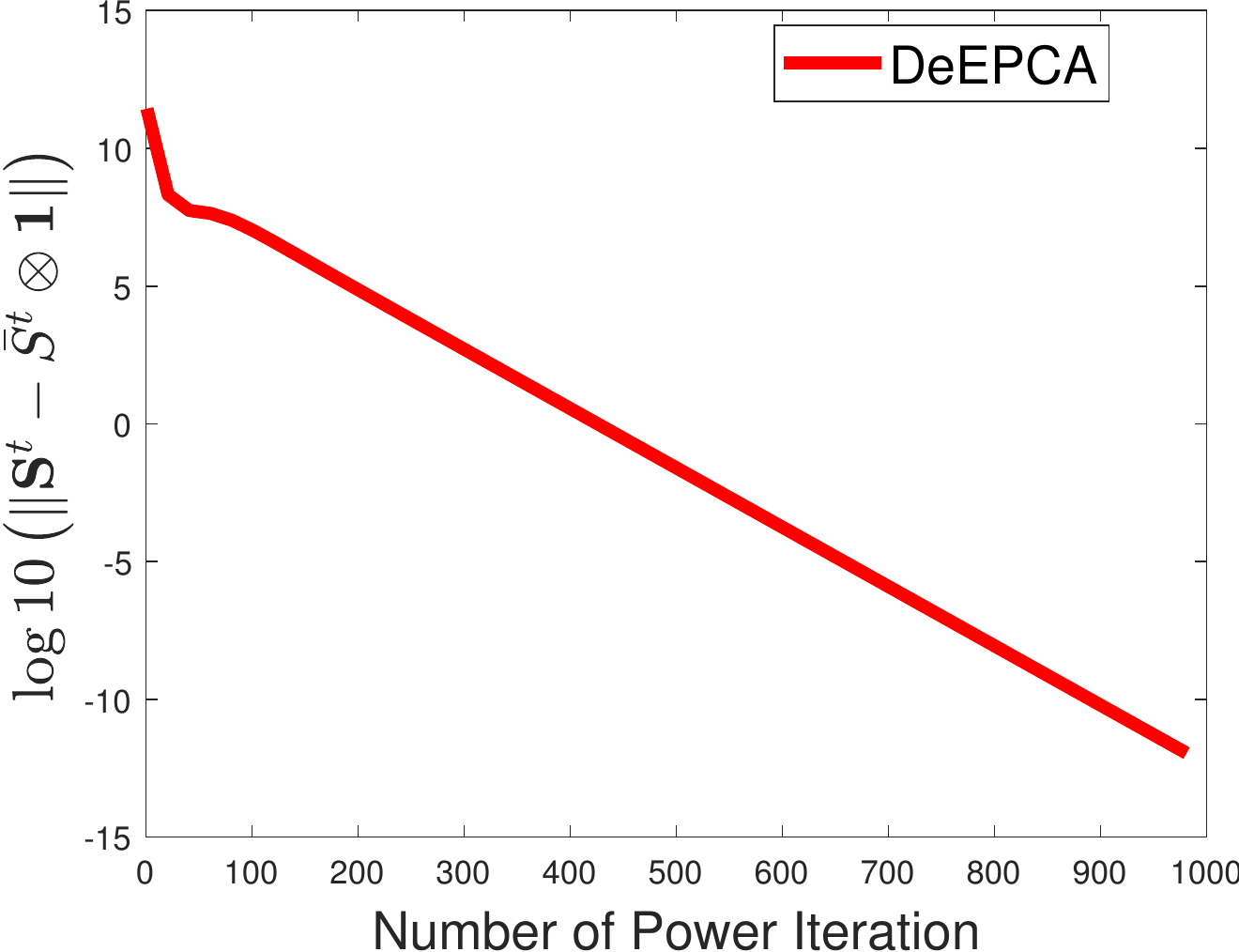}}~
	\subfigure[$\norm{\Wb - \bbW\otimes \mathbf{1}}$ with $K=1$]{\includegraphics[width=50mm]{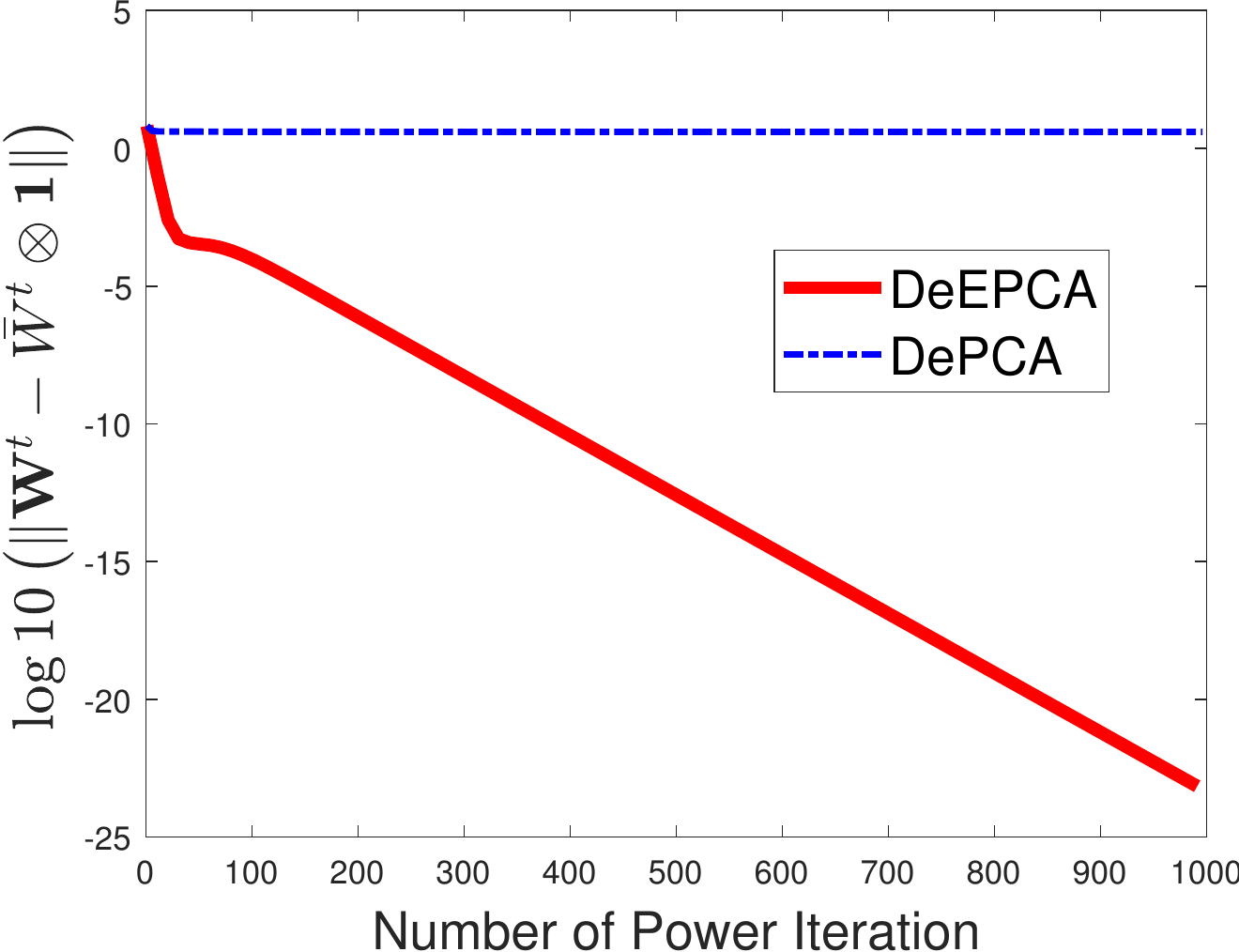}}~
	\subfigure[$\tan\theta_k(U, W)$ with $K=1$]{\includegraphics[width=50mm]{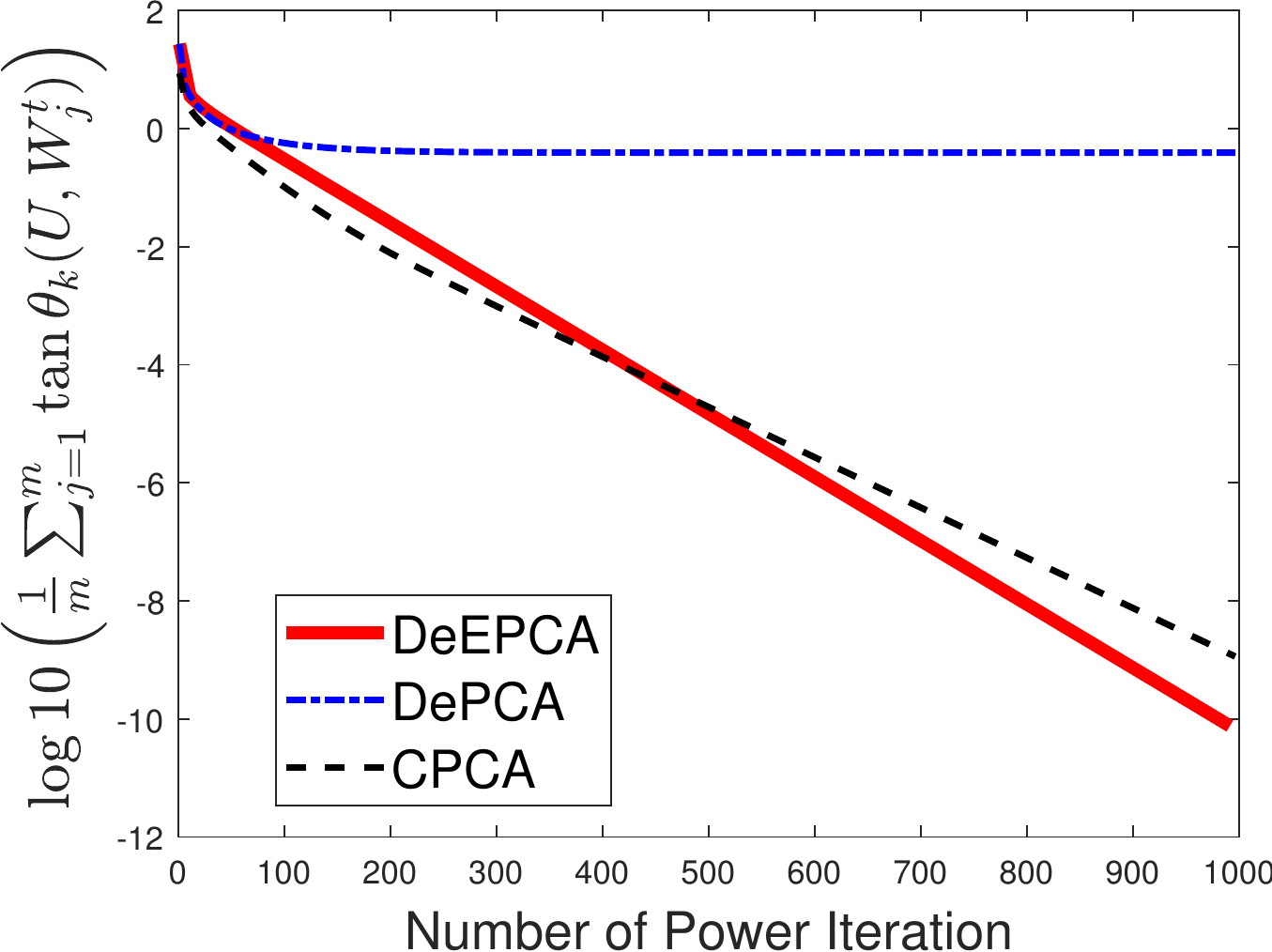}}
	\\
	\subfigure[$\norm{\Sb - \bbS\otimes \mathbf{1}}$ with $K = 5$]{\includegraphics[width=50mm]{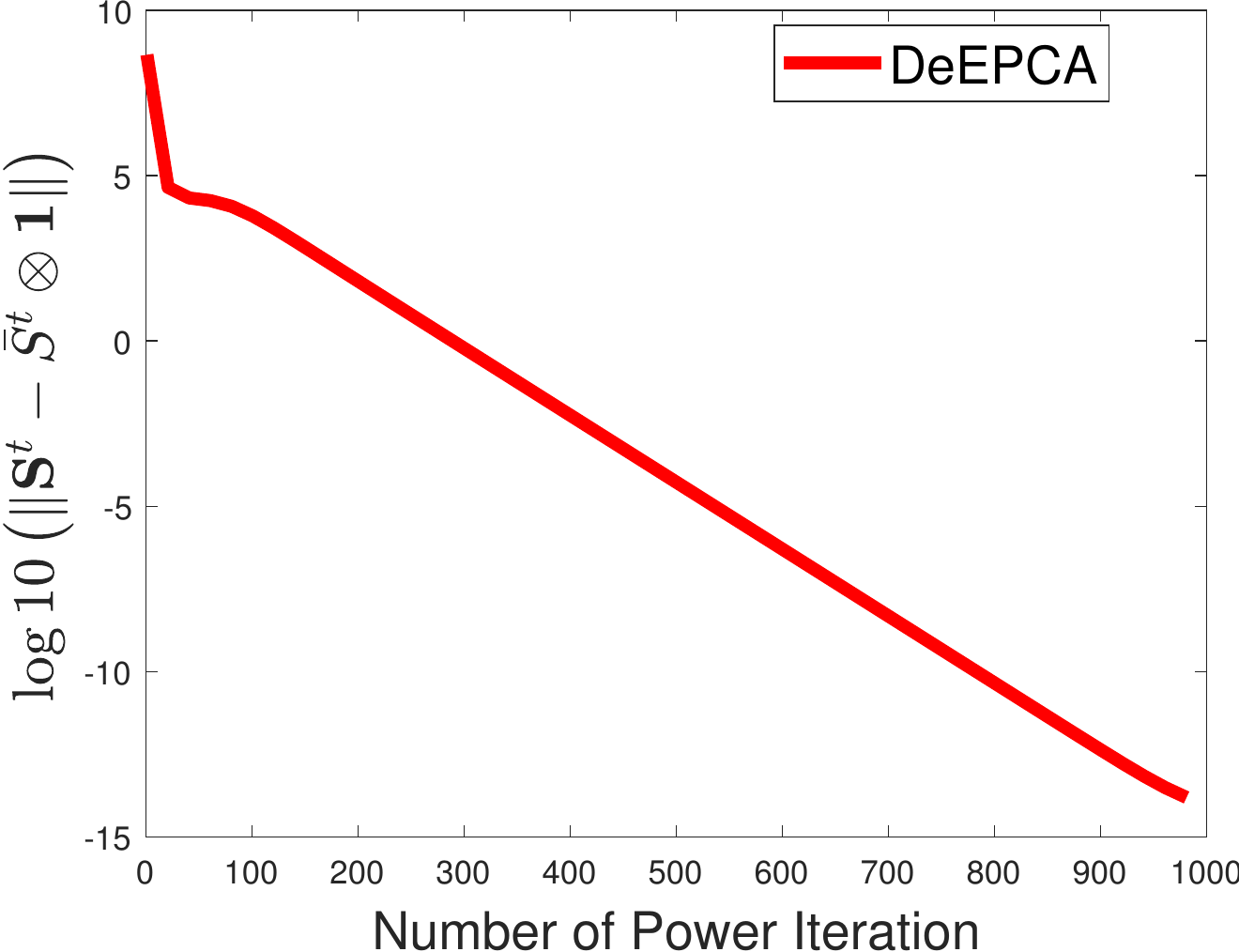}}~
	\subfigure[$\norm{\Wb - \bbW\otimes \mathbf{1}}$ with $K=5$]{\includegraphics[width=50mm]{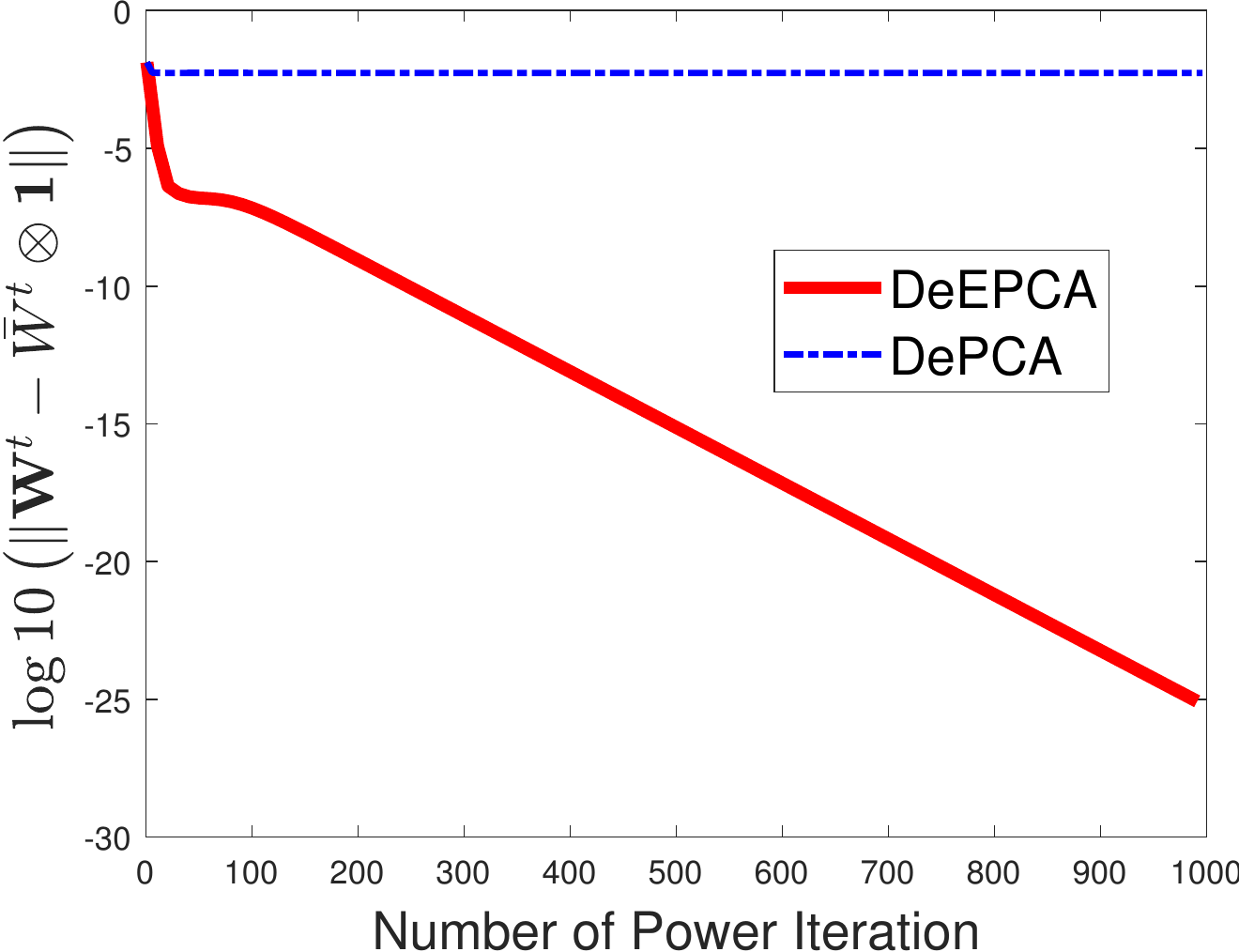}}~
	\subfigure[$\tan\theta_k(U, W)$ with $K=5$]{\includegraphics[width=50mm]{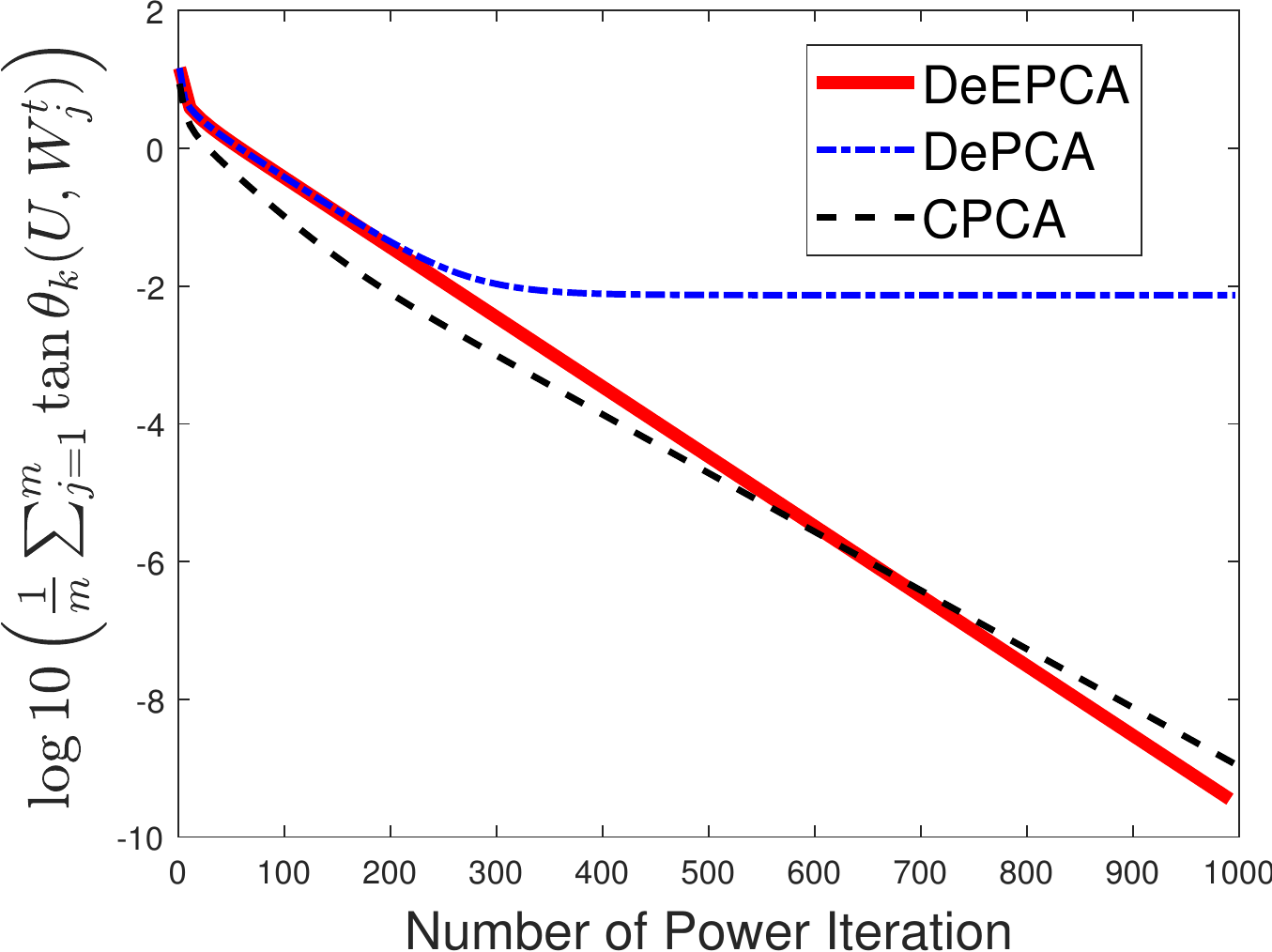}}
	\\
	\subfigure[$\norm{\Sb - \bbS\otimes \mathbf{1}}$ with $K = 10$]{\includegraphics[width=50mm]{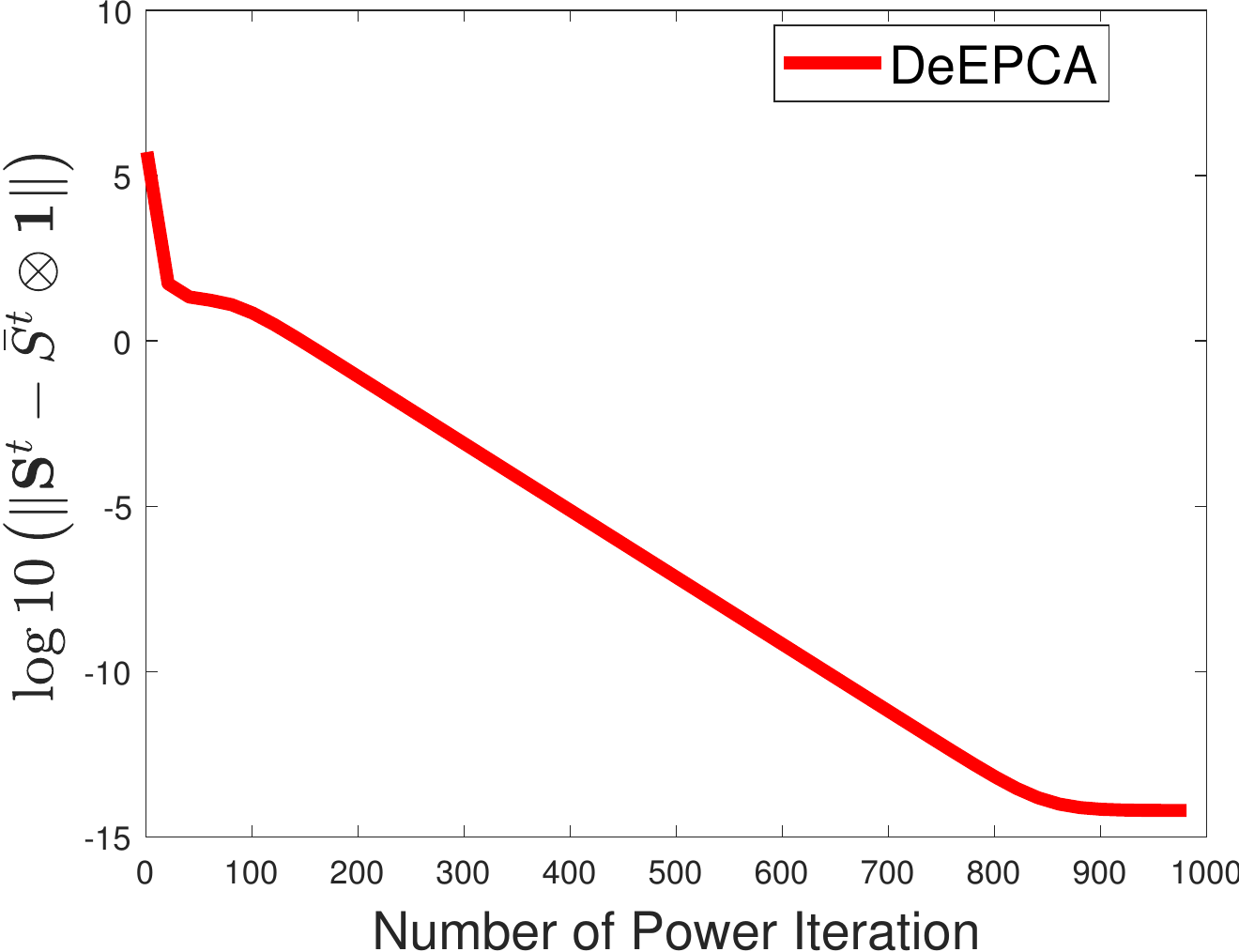}}~
	\subfigure[$\norm{\Wb - \bbW\otimes \mathbf{1}}$ with $K=10$]{\includegraphics[width=50mm]{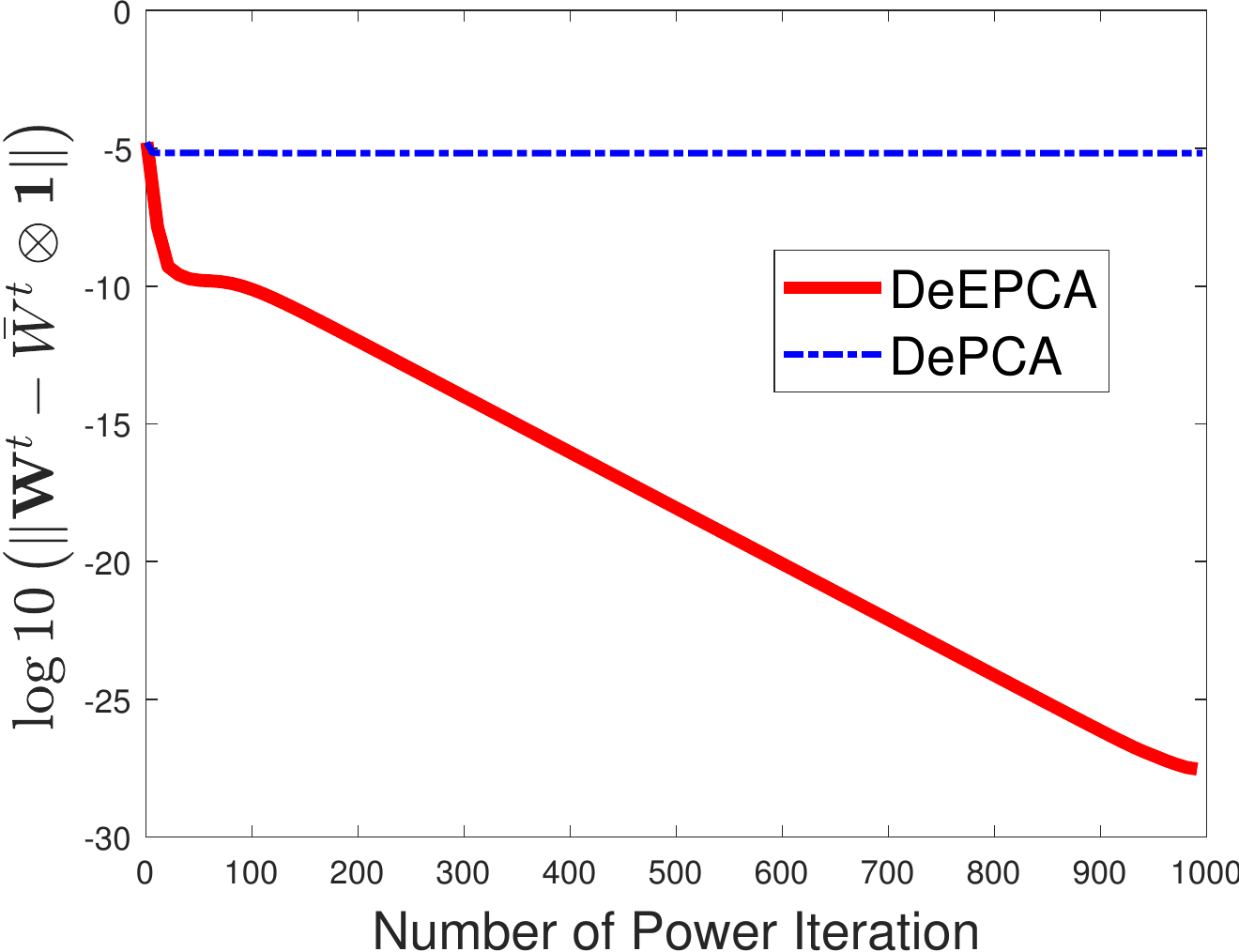}}~
	\subfigure[$\tan\theta_k(U, W)$ with $K=10$]{\includegraphics[width=50mm]{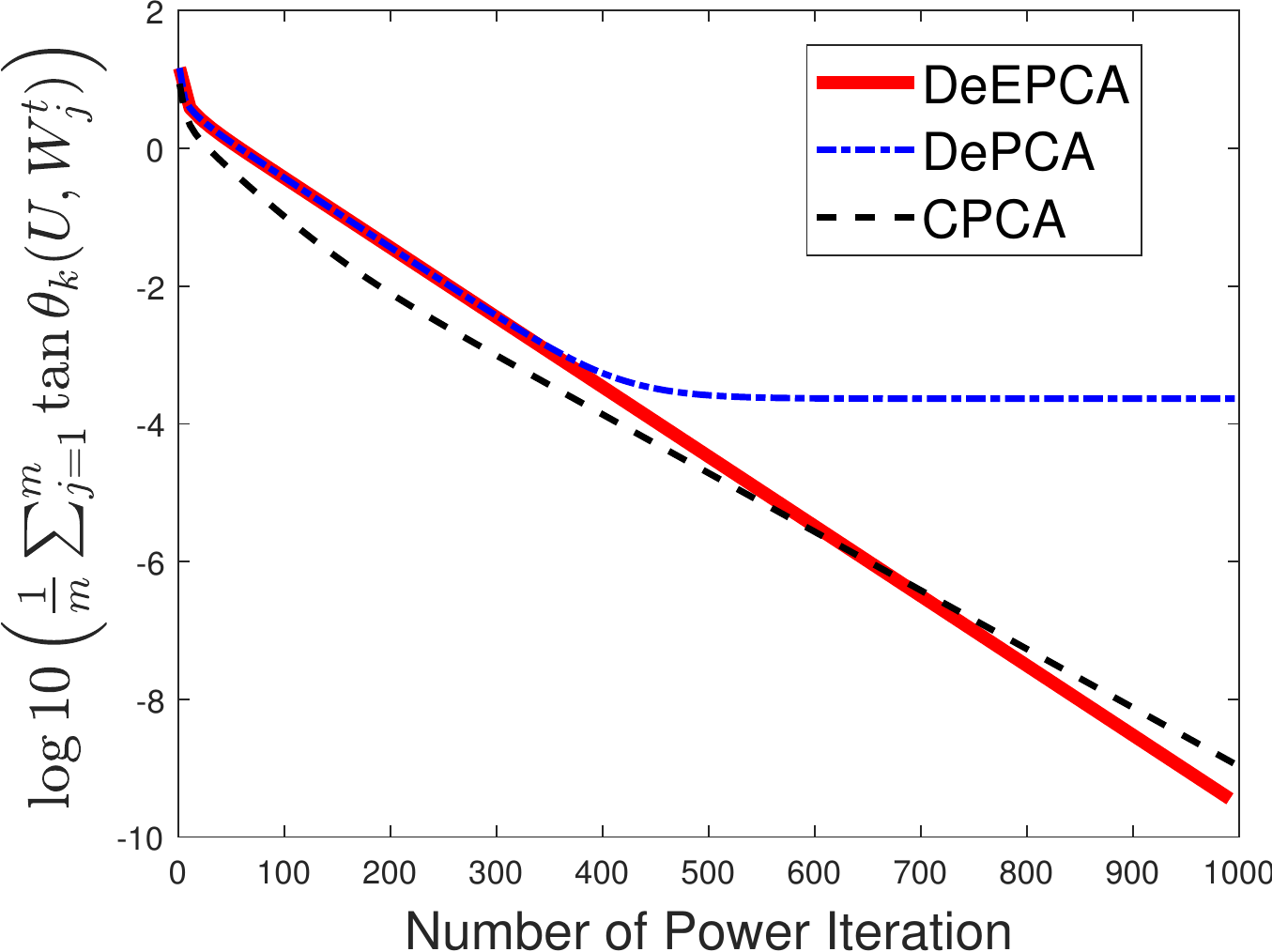}}
\end{center}
\vskip -0.1in
\caption{Experiment on `a9a'.}
\label{fig:a9a}
\vskip -0.2in
\end{figure*}

In our experiments, we compare \texttt{DeEPCA} with decentralized PCA (\texttt{DePCA}) \citep{wai2017fast}, and centralized PCA (\texttt{CPCA}).
We will study how consensus steps affect the convergence rate of \texttt{DeEPCA} empirically. 
Thus, we set different $K$'s in our experiment,. 
We will report the convergence rate of $\norm{\Sb^t - \bbS^t\otimes \mathbf{1}}$, $\norm{\Wb^t - \bbW^t\otimes\mathbf{1}}$ and $\frac{1}{m}\sum_{j=1}^m\tan\theta_k(U,W_j^t)$.
We report experiment results in Figure~\ref{fig:w8a} and Figure~\ref{fig:a9a}.

Figure~\ref{fig:w8a} shows that multi-consensus step is required in our \texttt{DeEPCA}.
When $K = 3$, \texttt{DeEPCA} can not converge to the top-$k$ principal components of $A$. 
The number of consensus steps of \texttt{DeECPA} in each power iteration should be determined by the heterogeneity of the data just as discussed in Remark~\ref{rmk:main}.
Furthermore, once consensus steps of \texttt{DeECPA} are sufficient, then \texttt{DeEPCA} can achieve a fast convergence rate comparable to centralized PCA which can be observed from Figure~\ref{fig:w8a} and Figure~\ref{fig:a9a}.
This validates our convergence analysis of \texttt{DeEPCA} in Theorem~\ref{thm:main}. 

Figure~\ref{fig:w8a} and Figure~\ref{fig:a9a} show that without increasing consensus steps, \texttt{DePCA} can not converge to the top-$k$ principal components of $A$.
Because of lacking of subspace tracking, to achieve a high precision solution, \texttt{DePCA} can only depends on an increasing consensus steps which can be observed from third columns of Figure~\ref{fig:w8a} and Figure~\ref{fig:a9a}.
Comparing \texttt{DeEPCA} and \texttt{DePCA}, we can conclude that \texttt{DeEPCA} has great advantages in communication cost.

\section{Conclusion}

This paper proposed a novel decentralized PCA algorithm \texttt{DeEPCA} that can achieve a linear convergence rate similar to the centralized PCA method, and the number of communications per multi-consensus step does \emph{not} depend on the target precision $\epsilon$. 
In this way, \texttt{DeEPCA} can achieve the best known communication complexity for decentralized PCA. 
Our experiments also verifies the communication efficiency of \texttt{DeEPCA}.
Although the analysis of \texttt{DeEPCA} is based on undirected graph and `FastMix', it can be easily extended to handle directed graphs because our analysis of \texttt{DeEPCA} only requires averaging. As a final remark,
we note that \texttt{DeEPCA} employs the power method, which can be applied to eigenvector finding, low rank matrix approximation, spectral analysis, etc. 
Therefore \texttt{DeEPCA} can be used to design communication efficient decentralized algorithms for these problems as well.

\bibliography{ref.bib}
\bibliographystyle{apalike2}

\pagebreak
\appendix

\section{Proof of Lemmas in Section~\ref{subsec:main_lem}}

We will prove our lemmas in the order of their appearance.

\subsection{Proof of Lemma~\ref{lem:s_g}}
\begin{proof}[Proof of Lemma~\ref{lem:s_g}]
First, because the operation `FastMix' is linear, we can obtain that
\begin{align*}
	\bbS^{t+1} = \bbS^t + \bbG^{t+1} - \bbG^t.
\end{align*}
We prove the result by induction. When $t = 0$, it holds that $\bbS^0 = \bbG^0 = W^0$.
Supposing it holds that $\bbS^t = \bbG^t$, then we have
\begin{align*}
	\bbS^{t+1} 
	= 
	\bbS^t + \bbG^{t+1} - \bbG^t
	=\bbG^{t+1}.
\end{align*}
Thus, for each $t = 0,1,\dots$, it holds that $\bbS^t = \bbG^t$.
\end{proof}

\subsection{Proof of Lemma~\ref{lem:g_h}}

\begin{proof}[Proof of Lemma~\ref{lem:g_h}]
By the definition of $\bbG^t$ and $\bbH^t$ in Eqn.~\eqref{eq:def_bb}, we have
\begin{align*}
	\norm{\frac{1}{m} \sum_{j=1}^m A_j W_j^{t-1} - \frac{1}{m}\sum_{j=1}^m A_j \bbW^{t-1}}^2
	\le&
	\frac{1}{m} \sum_{j=1}^m \norm{A_j (W_j^{t-1} - \bbW^{t-1})}^2
	\\
	\le&
	\frac{1}{m} \sum_{j=1}^m \norm{A_j}_2^2\cdot \norm{W_j^{t-1} - \bbW^{t-1}}^2
	\\
	\le&
	\frac{L^2}{m}\sum_{j=1}^m \norm{W_j^{t-1} - \bbW^{t-1}}^2
	\\
	=&
	\frac{L^2}{m}\norm{\Wb^{t-1} -  \bbW^{t-1}\otimes \mathbf{1}}^2,
\end{align*}
where the last inequality is because of the assumption $\norm{A_j}_2 \le L$ for $j = 1,\dots, m$.
\end{proof}

\subsection{Proof of Lemma~\ref{lem:S_s}}
\begin{proof}[Proof of Lemma~\ref{lem:S_s}]
For notation convenience, we use $\TB(\Wb)$ to denote the `FastMix' operation on $\Wb$, which is used in Algorithm~\ref{alg:deepca}. That is, 
\begin{equation*}
	\TB(\Wb) \triangleq \mbox{FastMix}(\Wb, K).
\end{equation*}
Then for $\Wb$, it holds that
\begin{align}
	\label{eq:rho_dec}
	\norm{\TB(\Wb) - \bbW\otimes \mathbf{1}} 
	\le
	\rho \cdot \norm{\Wb - \bbW\otimes \mathbf{1}}.
\end{align}
It is obvious that the `FastMix' operation $\TB(\cdot)$ is linear.
By the update rule of $\Sb^t$, we have
\begin{align*}
	\norm{\Sb^{t+1} -  \bbS^{t+1}\otimes \mathbf{1}}
	\overset{\eqref{eq:SS}}{=}&\norm{\TB(\Sb^t + \Gb^{t+1} - \Gb^t) -  \left(\bbS^{t+1} + \bbG^{t+1} -\bbG^t\right)\otimes \mathbf{1}} 
	\\
	\overset{\eqref{eq:rho_dec}}{\le}&
	\rho \norm{\Sb^t - \bbS^t\otimes\mathbf{1}} + \rho \norm{\Gb^{t+1} - \Gb^t - (\bbG^{t+1} - \bbG^t)\otimes \mathbf{1}}
	\\
	\le&
	\rho \norm{\Sb^t - \bbS^t\otimes\mathbf{1}} + \rho \norm{\Gb^{t+1} - \Gb^t}
	\\
	=&
	\rho \norm{\Sb^t - \bbS^t\otimes\mathbf{1}} 
	+ \rho \sqrt{
		\sum_j^m\norm{A_j (W_j^t - W_j^{t-1})}^2
	}
	\\
	\le& \rho \norm{\Sb^t -  \bbS^t\otimes \mathbf{1}} + L\rho\norm{\Wb^t - \Wb^{t-1}}.
\end{align*}
where the second inequality is because the fact that for any $\Wb\in\RR^{d\times k\times m}$, it holds that
\begin{align*}
	\norm{\Wb - \bbW\otimes \mathbf{1}}^2 
	=& 
	\sum_{j=1}^m \norm{W_j - \frac{1}{m}\sum_{i=1}^m W_i}^2
	\\
	=&
	\sum_{j=1}^m \norm{W_j}^2 + \norm{\frac{1}{m}\sum_{i=1}^m W_i}^2 - 2\sum_{j=1}^m \dotprod{W_j, \frac{1}{m}\sum_{i=1}^m W_i}
	\\
	=&\sum_{j=1}^m \norm{W_j}^2 - \norm{\frac{1}{m}\sum_{i=1}^m W_i}^2
	\\
	\le&
	\sum_{j=1}^m \norm{W_j}^2
	\\=&
	\norm{\Wb}^2.
\end{align*}
The last inequality is because of
\begin{align*}
	\sum_j^m\norm{A_j (W_j^t - W_j^{t-1})}^2
	\le
	\sum_{j=1}^m \norm{A_j}_2^2\cdot \norm{W_j^t - W_j^{t-1}}^2
	\le
	L^2 \sum_{j=1}^m\norm{W_j^t - W_j^{t-1}}
	=
	L^2 \norm{\Wb^t - \Wb^{t-1}}^2.
\end{align*}
\end{proof}

\subsection{Proof of Lemma~\ref{lem:bbs_norm}}
\begin{proof}[Proof of Lemma~\ref{lem:bbs_norm}]
By the definition of $\sigma_{\min}(\bbS^{t+1})$ and Lemma~\ref{lem:s_g}, we can obtain
\begin{align*}
	\sigma_{\min}(\bbS^{t+1})
	=&
	\sigma_{\min}(\bbG^{t+1})
	\ge
	\sigma_{\min}(\bbH^{t+1}) 
	- 
	\norm{\bbH^{t+1} - \bbG^{t+1}}
	\\
	=&
	\sigma_{\min}(A\bbW^t) 
	- 
	\norm{\bbH^{t+1} - \bbG^{t+1}}
	\\
	\ge&
	\sigma_{\min}(A\tW^t) 
	-
	\norm{A(\tW^t - \bbW^t)}
	- 
	\norm{\bbH^{t+1} - \bbG^{t+1}}
	\\
	\ge&
	\sigma_{\min}(A\tW^t)
	-
	L\norm{\tW^t - \bbW^t}
	-
	\norm{\bbH^{t+1} - \bbG^{t+1}}
	\\
	\overset{\eqref{eq:g_h},\eqref{eq:W_w},\eqref{eq:w_s}}{\ge}&
	\sigma_{\min}(A\tW^t)
	-
	\frac{24L}{\sqrt{m}} \norm{\left[\bbS^t\right]^\dagger}\norm{\Sb^t -  \bbS^t\otimes \mathbf{1}}.
\end{align*}
Furthermore, we have
\begin{align*}
	\sigma_{\min}(A\tW^t)
	=&
	\sigma_{\min}\left(
	\begin{bmatrix}
		\Sigma_k U^\top \tW^t\\
		\Sigma_{\setminus k} V^\top \tW^t
	\end{bmatrix}
	\right)
	\ge
	\sigma_{\min}\left(
	\begin{bmatrix}
		\Sigma_k U^\top \tW^t\\
		\mathbf{0}
	\end{bmatrix}
	\right)
	\\
	\ge&
	\lambda_k \cdot \sigma_{\min}(U^\top \tW^t)
	\overset{\eqref{eq:theta_def_1}}{=}
	\lambda_k\cdot \cos\theta_k(U, \tW^t)
	\\
	=&
	\lambda_k\cdot \frac{1}{\sqrt{1+ \mL^2(\bbS^t)}},
\end{align*}
where the first inequality is because of Corollary 7.3.6 of \cite{horn2012matrix} and matrix $\Sigma_k U^\top \tW^t$ is non-singular.

Therefore, we can obtain
\begin{align*}
	\sigma_{\min}(\bbS^{t+1}) 
	\ge
	\lambda_k\cdot \frac{1}{\sqrt{1+ \mL^2(\bbS^t)}}
	-
	\frac{24L}{\sqrt{m}} \norm{\left[\bbS^t\right]^\dagger}\norm{\Sb^t -  \bbS^t\otimes \mathbf{1}}.
\end{align*}
\end{proof}

\subsection{Proof of Lemma~\ref{lem:W_w}}

First, we give a important lemma that will be used in our proof.
\begin{lemma}[Theorem 3.1 of \cite{stewart1977perturbation}]
Let $A=QR$, where $A\in\RR^{d\times k}$ has rank $k$ and $Q^\top Q = I$ with $I$ being the identity matrix. 
Let $E$ satisfy $\|A^\dagger\| \norm{E} <\frac{1}{2}$ where $A^\dagger$ is the pseudo inverse of $A$.
Moreover $A+E = (Q+\Delta_Q) (R + \Delta_R)$, where $Q+\Delta_Q$ has orthogonal columns.
Then it holds that
\begin{align}
	\label{eq:Delta_Q}
	\norm{\Delta_Q} \le \frac{3 \|A^\dagger\|\norm{E}}{1 - 2 \norm{A^\dagger} \norm{E}}.
\end{align}
\end{lemma}

\begin{proof}[Proof of Lemma~\ref{lem:W_w}]
For notation convenience, we will omit the superscript.
Let $S_j = W_j R_j$ and $\bbS = \tW \tR$ be the QR decomposition of $S_j$ and $\bbS$, respectively .
Then we have
\begin{align*}
	&\norm{\Wb -  \bbW\otimes\mathbf{1}}^2
	\\
	=&
	\sum_{j=1}^m \norm{W_j - \frac{1}{m}\sum_{i=1}^m W_i}^2
	\le
	2 \sum_{j=1}^m \norm{W_j - \tW}^2 + 2 m\norm{\tW - \frac{1}{m}\sum_{i=1}^m W_i}^2
	\\
	\le&
	4 \sum_{j=1}^m \norm{W_j - \tW}^2
	\\
	\overset{\eqref{eq:Delta_Q}}{\le}&
	4 \sum_{j=1}^m \left(\frac{3\norm{\bbS^{\dagger}} \norm{\bbS - S_j}}{ 1- 2\norm{\bbS^{\dagger}} \norm{\bbS - S_j}}\right)^2
	\\
	\le&
	(12)^2 \cdot \norm{\bbS^\dagger}^2 \norm{\Sb - \bbS\otimes \mathbf{1}}^2,
\end{align*}
where the last inequality is because of the assumption $\norm{\bbS^\dagger} \norm{\bbS - S_j} \le \frac{1}{4}$.
Hence, we can obtain that
\begin{align*}
	\norm{\Wb -  \bbW\otimes\mathbf{1}}
	\le
	12 \norm{\bbS^\dagger} \norm{\Sb - \bbS\otimes \mathbf{1}}.
\end{align*}
\end{proof}

\subsection{Proof of Lemma~\ref{lem:mL_dec}}

\begin{proof}[Proof of Lemma~\ref{lem:mL_dec}]
By the update rule of Algorithm, we can obtain that
\begin{align*}
	\mL(\bbS^{t+1})
	&=
	\tan\theta_k(U, \bbS^{t+1}) 
	\\
	=&\max_{\norm{w} = 1}
	\frac{
		\norm{V^\top \bbS^{t+1}w}}{
		\norm{U^\top \bbS^{t+1} w}
	}
	=\max_{\norm{w} = 1}
	\frac{
		\norm{V^\top \bbG^{t+1}w}}{
		\norm{U^\top \bbG^{t+1} w}
	}
	\\
	\le&
	\max_{\norm{w} = 1}
	\frac{
		\norm{V^\top \bbH^{t+1}w} + \norm{\bbG^{t+1} - \bbH^{t+1}}}{
		\norm{U^\top \bbH^{t+1} w} - \norm{\bbG^{t+1} - \bbH^{t+1}}
	}
	\\
	\overset{\eqref{eq:g_h}}{\le}&
	\max_{\norm{w} = 1}
	\frac{
		\norm{V^\top \bbH^{t+1}w} + \frac{L}{\sqrt{m}} \norm{\Wb^t - \bbW^t\otimes \mathbf{1}}}{
		\norm{U^\top \bbH^{t+1} w} - \frac{L}{\sqrt{m}} \norm{\Wb^t - \bbW^t\otimes \mathbf{1}}
	}
	\\
	=&
	\max_{\norm{w} = 1}
	\frac{
		\norm{V^\top A \bbW^t w} + \frac{L}{\sqrt{m}} \norm{\Wb^t - \bbW^t\otimes \mathbf{1}}}{
		\norm{U^\top A \bbW^t w} - \frac{L}{\sqrt{m}} \norm{\Wb^t - \bbW^t\otimes \mathbf{1}}
	}
	\\
	\le&
	\max_{\norm{w} = 1}
	\frac{
		\lambda_{k+1}\norm{V^\top \bbW^t w} + \frac{L}{\sqrt{m}} \norm{\Wb^t - \bbW^t\otimes \mathbf{1}}}{
		\lambda_k\norm{U^\top \bbW^t w} - \frac{L}{\sqrt{m}} \norm{\Wb^t - \bbW^t\otimes \mathbf{1}}
	}
	\\
	\le&
	\max_{\norm{w} = 1}
	\frac{
		\lambda_{k+1}\norm{V^\top \tW^t w} + \lambda_{k+1} \norm{\tW^t - \bbW^t} + \frac{L}{\sqrt{m}} \norm{\Wb^t - \bbW^t\otimes \mathbf{1}}}{
		\lambda_k\norm{U^\top \tW^t w} - \lambda_k \norm{\tW^t - \bbW^t} - \frac{L}{\sqrt{m}} \norm{\Wb^t - \bbW^t\otimes \mathbf{1}}
	}
	\\
	\overset{\eqref{eq:w_s},\eqref{eq:W_w}}{\le}&
	\max_{\norm{w} = 1}
	\frac{
		\lambda_{k+1}\norm{V^\top \tW^t w} + \frac{12(\lambda_{k+1}+L)}{\sqrt{m}} \norm{\left[\bbS^t\right]^\dagger}\norm{\Sb^t -  \bbS^t\otimes \mathbf{1}}}{
		\lambda_k\norm{U^\top \tW^t w} - \frac{12(\lambda_k+L)}{\sqrt{m}} \norm{\left[\bbS^t\right]^\dagger}\norm{\Sb^t -  \bbS^t\otimes \mathbf{1}}
	}
	\\
	=&
	\max_{\norm{w} = 1}
	\frac{
		\lambda_{k+1}\norm{V^\top \tW^t w}/\norm{U^\top \tW^t w} + \frac{12(\lambda_{k+1}+L)}{\sqrt{m}} \norm{\left[\bbS^t\right]^\dagger}\norm{\Sb^t -  \bbS^t\otimes \mathbf{1}}/\norm{U^\top \tW^t w}}{
		\lambda_k - \frac{12(\lambda_k+L)}{\sqrt{m}} \norm{\left[\bbS^t\right]^\dagger}\norm{\Sb^t -  \bbS^t\otimes \mathbf{1}}/\norm{U^\top \tW^t w}
	}.
\end{align*}
Furthermore, we have
\begin{align*}
	\frac{1}{\norm{U^\top \tW^t w}}
	\le 
	\max_{\norm{w} = 1}\frac{1}{\norm{U^\top \tW^t w}} 
	=
	\frac{1}{\cos\theta_k(U, \tW^t)}.
\end{align*}
Thus, we can obtain that
\begin{align}
	\mL(\bbS^{t+1})
	\le&
	\max_{\norm{w} = 1}
	\frac{
		\lambda_{k+1}\norm{V^\top \tW^t w}/\norm{U^\top \tW^t w} 
		+ 
		\frac{12(\lambda_{k+1}+L)}{\sqrt{m}} \norm{\left[\bbS^t\right]^\dagger}\norm{\Sb^t -  \bbS^t\otimes \mathbf{1}}/\cos\theta_k(U, \tW^t)
	}{\lambda_k
		-
		\frac{12(\lambda_k+L)}{\sqrt{m}} \norm{\left[\bbS^t\right]^\dagger}\norm{\Sb^t -  \bbS^t\otimes \mathbf{1}}/\cos\theta_k(U, \tW^t)
	}\notag
	\\
	=&
	\frac{
		\lambda_{k+1} \mL(\bbS^t)
		+
		\frac{12(\lambda_{k+1}+L)}{\sqrt{m}} \norm{\left[\bbS^t\right]^\dagger}\norm{\Sb^t -  \bbS^t\otimes \mathbf{1}} \cdot \sqrt{1+\mL^2(\bbS^t)}
	}{\lambda_k - 
		\frac{12(\lambda_k+L)}{\sqrt{m}} \norm{\left[\bbS^t\right]^\dagger}\norm{\Sb^t -  \bbS^t\otimes \mathbf{1}} \cdot \sqrt{1+\mL^2(\bbS^t)}
	}, \label{eq:update}
\end{align}
where the last equality is because of the fact $1 + \tan^2\theta = \frac{1}{\cos^2\theta}$.

Now we will prove the result by induction. 
When $t= 0$, it holds that $S_j^0$'s are equal to each other, that is, $\norm{\Sb^0 - \bbS^0\otimes \mathbf{1}} = 0$.
Hence, we can obtain that
\begin{align*}
	\mL(\bbS^1) 
	\le
	\frac{\lambda_{k+1}}{\lambda_k} \mL(\bbS^0) 
	<\left(1 - \frac{\lambda_k - \lambda_{k+1}}{2\lambda_k}\right) \cdot \mL(\bbS^0).
\end{align*}
We assume that $\mL(\bbS^t) \le \gamma^t \cdot \mL(\bbS^0)$ and Eqn.~\eqref{eq:s_ass} hold.
Replacing the assumptions to Eqn.~\eqref{eq:update}, we can obtain that
\begin{align*}
	\mL(\bbS^{t+1}) \le \left(1 - \frac{\lambda_k - \lambda_{k+1}}{2\lambda_k}\right)^{t+1} \cdot\mL(\bbS^0) 
	= \gamma^{t+1} \cdot\mL(\bbS^0) .
\end{align*}
This concludes the proof.
\end{proof}

\subsection{Proof of Lemma~\ref{lem:W-W}}

\begin{proof}[Proof of Lemma~\ref{lem:W-W}]
First, by triangle inequality, we can obtain
\begin{align*}
	\norm{\Wb^t - \Wb^{t-1}}
	\le&
	\norm{\Wb^t - \bbW^t\otimes \mathbf{1}}
	+
	\norm{\Wb^{t-1} - \bbW^{t-1}\otimes \mathbf{1}}
	+
	\norm{\bbW^t\otimes \mathbf{1} - \bbW^{t-1}\otimes \mathbf{1}}
	\\
	\overset{\eqref{eq:W_w}}{\le}&
	12\left(
	\norm{\left[\bbS^t\right]^\dagger}\norm{\Sb^t -  \bbS^t\otimes \mathbf{1}}
	+
	\norm{\left[\bbS^{t-1}\right]^\dagger}\norm{\Sb^{t-1} -  \bbS^{t-1}\otimes \mathbf{1}}
	\right)
	+
	\sqrt{m}\norm{\bbW^t - \bbW^{t-1}}.
\end{align*}
Furthermore, we have
\begin{align*}
	\norm{\bbW^t - \bbW^{t-1}}
	\le&
	\norm{
		\bbW^t - U
	}
	+
	\norm{\bbW^{t-1} - U}
	\\
	\le&
	\norm{\tW^t - U} + \norm{\tW^t - \bbW^t}
	+
	\norm{\tW^{t-1} - U} + \norm{\tW^{t-1} - \bbW^{t-1}}
	\\
	\overset{\eqref{eq:w_s}}{\le}&
	\norm{\tW^t - U} + \norm{\tW^{t-1} - U}
	\\&
	+\frac{12}{\sqrt{m}} 
	\left(
	\norm{\left[\bbS^t\right]^\dagger}\norm{\Sb^t -  \bbS^t\otimes \mathbf{1}}
	+
	\norm{\left[\bbS^{t-1}\right]^\dagger}\norm{\Sb^{t-1} -  \bbS^{t-1}\otimes \mathbf{1}}
	\right).
\end{align*}

Now we begin to bound the value of $\norm{\tW^t - U}$. Note that due to sign adjustment in Eqn.~\eqref{eq:sa} in Algorithm~\ref{alg:deepca}, then $\Wb^t$ and $\Wb^{t-1}$ share the same direction, that is the dot product of columns of   $\Wb^t$ and $\Wb^{t-1}$ are positive. 
Thus, we can choose such $U$ that shares the same direction with   $\Wb^t$ and $\Wb^{t-1}$.
In this case, $\tW^t$ and $\tW^{t-1}$ can also  share the same direction with $U$.
Combining with  the definition of $\tW$ in Eqn.~\eqref{eq:def_bb}, we have
\begin{align*}
	\norm{\tW^t - U}^2 
	=&
	\norm{\tW^t}^2 + \norm{U}^2 - 2\dotprod{\tW^t, U} 
	\le 2k - 2k \cdot \sigma_{\min}(U^\top\tW^t)
	\\
	=&
	2k(1 - \cos\theta_k(\tW^t, U))
	=
	2k\left(1 - \frac{1}{\sqrt{1+\mL^2(\bbS^t)}}\right)
	\\
	=&2k \cdot \frac{ \sqrt{1+\mL^2(\bbS^t)} - 1}{\sqrt{1+\mL^2(\bbS^t)}}
	=
	2k\cdot  \frac{\mL^2(\bbS^t)}{\sqrt{1+\mL^2(\bbS^t)} (\sqrt{1+\mL^2(\bbS^t)}+1)}
	\\
	\le& k\cdot \mL^2(\bbS^t),
\end{align*}
where the first inequality is because of $U^\top(:,i) \tW^t(:,i) > 0$ and the inequality
\begin{align*}
	\dotprod{\tW^t, U}
	=
	\sum_{i=1}^k U^\top(:,i) \tW^t(:,i)
	\ge
	k \cdot \sigma_{\min}(U^\top \tW^t).
\end{align*}
Therefore, we can obtain that
\begin{align*}
	&\norm{\Wb^t - \Wb^{t-1}}
	\\
	\le&
	12\left(
	\norm{\left[\bbS^t\right]^\dagger}\norm{\Sb^t -  \bbS^t\otimes \mathbf{1}}
	+
	\norm{\left[\bbS^{t-1}\right]^\dagger}\norm{\Sb^{t-1} -  \bbS^{t-1}\otimes \mathbf{1}}
	\right)
	\\&
	+12\left(
	\norm{\left[\bbS^t\right]^\dagger}\norm{\Sb^t -  \bbS^t\otimes \mathbf{1}}
	+
	\norm{\left[\bbS^{t-1}\right]^\dagger}\norm{\Sb^{t-1} -  \bbS^{t-1}\otimes \mathbf{1}}
	\right)
	\\
	&+
	\sqrt{mk}\cdot \left(
	\mL(\bbS^t )
	+ 
	\mL(\bbS^{t-1})
	\right)
	\\
	=&
	24\left(
	\norm{\left[\bbS^t\right]^\dagger}\norm{\Sb^t -  \bbS^t\otimes \mathbf{1}}
	+
	\norm{\left[\bbS^{t-1}\right]^\dagger}\norm{\Sb^{t-1} -  \bbS^{t-1}\otimes \mathbf{1}}
	\right)
	+
	\sqrt{mk}\cdot \left(
	\mL(\bbS^t )
	+ 
	\mL(\bbS^{t-1})
	\right).
\end{align*}
\end{proof}

\end{document}